%% file: main.tex
\documentclass[journal]{IEEEtran}

\usepackage{amssymb,amsmath}

\usepackage{amsthm}
\usepackage{mathtools}
\usepackage{cancel}
\usepackage{multirow}
\usepackage{graphicx}
\usepackage{comment}
\usepackage[sort,compress,noadjust]{cite}
\usepackage[font=footnotesize]{subcaption}
\usepackage[font=footnotesize]{caption}

\usepackage{amsfonts}
\usepackage{tabularx}
\usepackage{graphicx}
\usepackage{array}
\usepackage{float}
\usepackage{hhline}
\usepackage{algorithm}
\usepackage{algorithmic}
\usepackage[colorlinks,urlcolor=blue]{hyperref}
\usepackage{setspace}
\usepackage{colortbl}
\usepackage[dvipsnames]{xcolor}
\usepackage[normalem]{ulem}
\usepackage{comment}
\usepackage{arydshln}
\DeclareMathOperator*{\argmax}{arg\,max}

\usepackage{accents}

\usepackage{placeins}

\theoremstyle{plain}
\newtheorem{theorem}{Theorem}
\theoremstyle{plain}
\newtheorem{proposition}{Proposition}
\theoremstyle{plain}
\newtheorem{lemma}{Lemma}
\theoremstyle{plain}
\newtheorem{corollary}{Corollary}
\theoremstyle{definition}
\newtheorem{assumption}{Assumption}
\theoremstyle{definition}
\newtheorem{definition}{Definition}
\theoremstyle{remark}

\usepackage[capitalize]{cleveref}
\crefformat{equation}{(#2#1#3)}
\Crefformat{equation}{Equation~(#2#1#3)}
\Crefname{equation}{Equation}{Eqs.}
\crefformat{assumption}{(#2#1#3)}
\Crefformat{assumption}{Assumption~(#2#1#3)}
\Crefname{assumption}{Assumption}{Assum.}

\usepackage[normalem]{ulem}

\begin{document}

\title{OptMap: Geometric Map Distillation via Submodular Maximization}

\author{David Thorne$^1$, Nathan Chan$^1$, Christa S. Robison$^2$, Philip R. Osteen$^2$, Brett T. Lopez$^1$
\thanks{*This research was sponsored by the DEVCOM Army Research Laboratory (ARL) under SARA CRA
W911NF-24-2-0017. Distribution Statement A: Approved for public release; distribution is unlimited.}
\thanks{$^1$ University of California, Los Angeles, Los Angeles, CA, USA {\tt\small \{davidthorne, nchan22, btlopez\}@ucla.edu}}%
\thanks{$^{2}$DEVCOM Army Research Laboratory (ARL), Adelphi, MD, USA. \{\texttt{christa.s.robison, philip.r.osteen \}.civ@army.mil}}}%



\maketitle

\addtolength{\topmargin}{0.05in}

\begin{abstract}
Autonomous robots rely on geometric maps to inform a diverse set of perception and decision-making algorithms.
As autonomy requires reasoning and planning on multiple scales, each algorithm may require a different map for optimal performance.
LiDAR sensors generate an abundance of geometric data (up to 50 MB per second) to satisfy these diverse requirements. 
However, the point-based operations required to process perception data are both memory and computationally expensive. 
Such operations can be bypassed via learned representations that encode similarity, but selecting informative, size-constrained maps remains an NP-hard combinatorial problem.
In this work we present OptMap: a geometric map distillation algorithm which achieves online, application-specific map generation via multiple theoretical and algorithmic innovations.
A central feature is the maximization of set functions that exhibit diminishing returns, i.e., submodularity, using polynomial-time algorithms with provably near-optimal solutions.
We formulate a novel submodular reward function which quantifies informativeness, reduces input set sizes, and minimizes solution bias.
Further, we propose a dynamically reordered streaming submodular algorithm which improves empirical solution quality and addresses input order bias via an online approximation of the value of all scans.
Testing was conducted on open-source and custom datasets with an emphasis on long-duration mapping sessions, highlighting OptMap's minimal computation requirements.
OptMap’s practical value is then illustrated through its application to online geometric change detection.
Open-source ROS1 and ROS2 packages are available and can be used alongside any LiDAR odometry algorithm.
\end{abstract}


\input{sections/Symbols}
\input{sections/introduction.tex} 
\input{sections/related_works.tex}

\input{sections/preliminaries.tex}
\input{sections/Optimal_LiDAR_Map_Distillation.tex}
\input{sections/Dynamically_Reordered_Streaming_Submodular_Maximization.tex}
\input{sections/Algorithmic_Implementation}
\input{sections/results.tex}
\input{sections/conclusion.tex}
\input{sections/Appendix.tex}

\IEEEpubidadjcol




\bibliographystyle{IEEEtran}
\bibliography{references}

\newpage

 


\vspace{11pt}

\vfill

\end{document}

%% file: sections/Symbols.tex
\section*{Symbols}
\label{sec:symbols}

\begin{table}[h]
\begin{tabular}{lp{2.45in}}
\hline
\textbf{Symbols} & \textbf{Meaning} \\
\hline
$E$, $R$, $S$ & Input, Reduced, and Solution Sets. \\[4pt]
$k$ & Cardinality constraint, i.e., $|S| \leq k$. \\[4pt]
$OPT$ & Optimal value, $f(x_{opt})=OPT$. \\[4pt]
$\mathbb{S}^{n}$, $\mathcal{A}(\mathbb{S}^{n})$ & $n-$hypersphere and its surface area. \\[4pt]
$e_{i} \in E$ & $i$'th input set element (descriptor). \\[4pt]
$r_{i} \in R$ & $i$'th reduced set element (descriptor). \\[4pt]
$\mathbf{0}$ & Null element, $\mathbf{0} = [0,0,...,0]$. \\[4pt]
$d(a,B)$ & Euclidean distance to a set, $d(a,B) = \min_{b \in B} ||a - b||_{2}$, $a,b \in \mathbb{R}^{n}$. \\[4pt]
$f(a|B)$ & Marginal value, $f(a|B) = f(B \cup a) - f(B)$. \\[2pt]
\hline
\end{tabular}
\end{table}

\newpage

%% file: sections/introduction.tex
\section{Introduction}
\label{sec:introduction}

\begin{figure}[t!]
    \centering
    \begin{subfigure}[b]{9cm}
        \centering
        \includegraphics[width=0.96\textwidth]{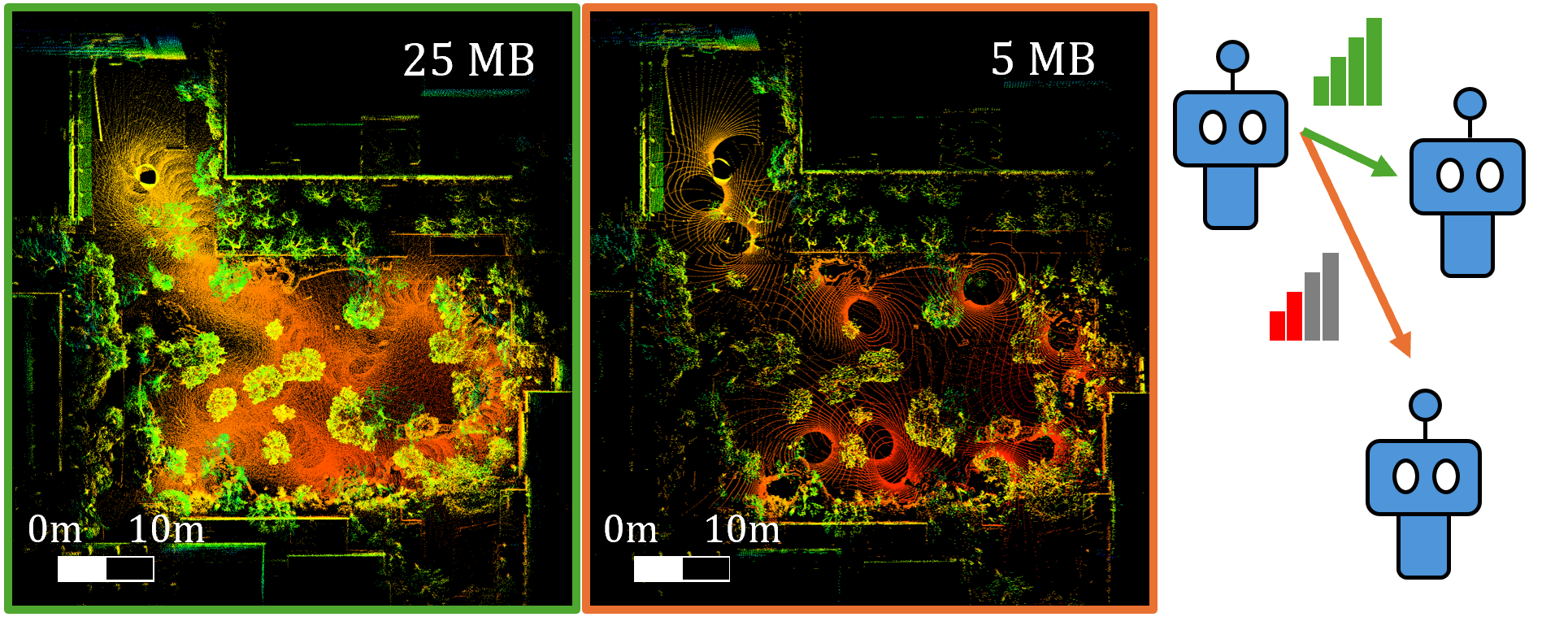}
        \caption{OptMap application maps for multi-robot map sharing. Maps can be generated to summarize a full session (Sculpture Garden) with a given size dictated by communication bandwidth constraints.}
        \label{fig:top-right-mapshare}
        \vspace{0.05in}
    \end{subfigure}\\
    \begin{subfigure}[b]{9cm}
        \centering
        \includegraphics[width=0.96\textwidth]{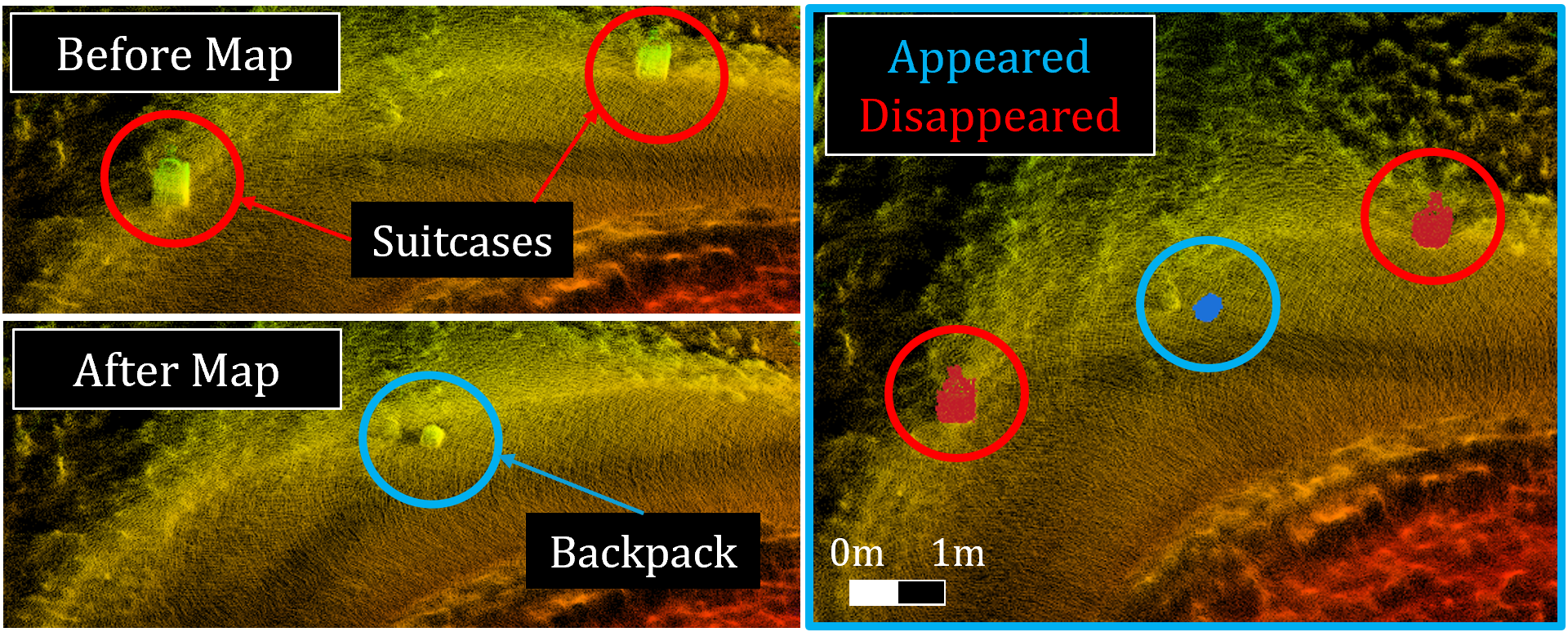}
        \caption{OptMap application maps for change detection (left). Online change detection demands concise yet dense maps. OptMap can provide these dense maps given a location or region query from a user.}
        \label{fig:top-right-changedetect}
    \end{subfigure}
    \caption{OptMap is a geometric map distillation algorithm which provides size-constrained, informative geometric maps for online autonomy algorithms.}
    \label{fig:top-right}
    \vskip -0.25in
\end{figure}

\begin{figure*}[t]
    \centering
    \vspace{6pt}
    \includegraphics[width=0.99\textwidth]{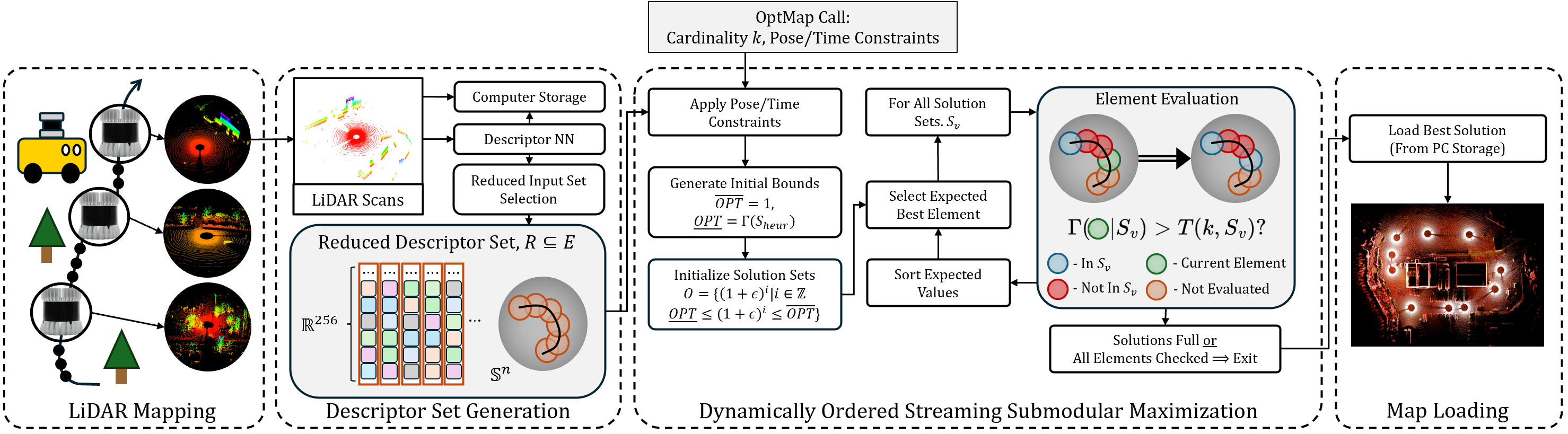}
    \caption{The OptMap pipeline consists of four stages: (i) LiDAR mapping (ii) descriptor set generation, (iii) dynamically ordered streaming submodular maximization, and (iv) map loading. Point clouds from LiDARs are used to generate descriptors which are then saved to computer storage. Redundant descriptors are then filtered to improve evaluation time. The dynamically ordered streaming submodular maximization uses an online approximate value of each element to address input order bias and select informative solutions.}
    \label{fig:OptMap_pipeline}
    \vskip -0.2in
\end{figure*}

\IEEEPARstart{M}{odern} autonomous systems use a modular software architecture with separate algorithms to perceive the environment, plan collision-free paths, estimate vehicle motion, and make high-level decisions. 
Many of these algorithms require a geometric map with \textit{unique requirements} to optimize performance and processing time.
For example, multi-robot teams need to share summary maps with variable size subject to communication bandwidth constraints (\cref{fig:top-right-mapshare}), but detecting small geometric changes in previously mapped areas requires concise yet dense maps (\cref{fig:top-right-changedetect}).
One approach to meet the requirements of different subsystems is to maintain a single point cloud map and use point operations, such as voxelization and nearest neighbor search, to provide application-specific maps.
However, this is prohibitively expensive from a memory and computation perspective because modern LiDAR sensors collect $250,000$-point scans at 20 Hz, or about 50 MB of data generated every second. 
Achieving the examples in \cref{fig:top-right} thus demands a new paradigm for generating application-specific maps using memory- and computationally-efficient learned representations.
These representations depict a point cloud as a high-dimensional vector and allow a single batch optimization to create an application map, removing the need to store the entire set in memory or perform point-based computations.
We call this problem \emph{geometric map distillation}, because key information gathered by perception sensors is extracted to create an application-specific map with maximum utility.
To this end, this paper proposes OptMap, an online geometric map distillation algorithm.

Geometric map distillation is most naturally formulated as a combinatorial optimization problem (i.e., optimal subset selection \cite{fujishige2005submodular}) that aims to maximize the informativeness of the output map given all available perception data.
Within this formulation, online mapping algorithms suffer from a causality constraint, as filtering-based approaches such as keyframe or occupancy maps must select solutions before future observations are available. 
For point clouds, informativeness is readily quantified using the amount of redundant data or overlapping regions, but optimizations that rely on direct overlap computations are expensive because geometric maps contain millions of points.
Solving such problems with formal suboptimality bounds is thus only possible as an offline process.
Previous research sought to simplify overlap calculations by learning to encode point cloud similarity via compact descriptor vectors \cite{ma2022overlaptransformer}.
Although most descriptors are used for global place recognition \cite{chen2022overlapnet, arce2023padloc, komorowski2021minkloc3d}, recent work used them to formulate submodular reward functions to summarize large LiDAR datasets \cite{thorne2025submodular}.
By combining descriptors with submodularity, or the mathematically defined property of diminishing returns, \cite{thorne2025submodular} leveraged several efficient submodular maximization algorithms for map summarization with no direct overlap computations.
The most famous such algorithm is the greedy algorithm, which guarantees a $(1-\frac{1}{e})$-suboptimality bound for NP-hard cardinality constrained submodular problems \cite{nemhauser1978analysis, nemhauser1978best}.
However, a recent class of streaming submodular algorithms is pertinent to optimizing over large datasets because they only require one pass over them \cite{badanidiyuru2014streaming, chekuri2015streaming}.
Streaming algorithms have been used for problems similar to geometric map distillation such as search engines and image dataset summarization \cite{feldman2018less, kazemi2019submodular} which both aim to select informative subsets from extremely large input sets.

In this paper, we leverage learned descriptors and streaming submodular maximization to develop an algorithm for real-time geometric map distillation called OptMap (\cref{fig:OptMap_pipeline}).
We propose a novel submodular function which assigns weights to elements based on their similarity to prior scans. 
This reduces bias towards over-sampled locations and allows for computation improvements by recognizing redundant geometric data.
This submodular function is optimized by a streaming submodular maximization algorithm that uses a novel dynamic reordering approach to improve solution quality and address multiple forms of input order bias.
Further computation time improvements are derived from our ability to use the simple geometric relationships between descriptors to derive tight \emph{a priori} solution bounds.
We demonstrate OptMap's ability to distill large LiDAR datasets in several settings, which emphasize high-quality output maps with minimal computational requirements.
Additionally, OptMap's utility is demonstrated for online change detection, where current and prior LiDAR maps are compared to identify objects that have appeared, disappeared, or moved.
OptMap is well-suited for this task as it can generate arbitrarily dense submaps on demand, unlike keyframes or downsampled dense global maps whose size and adaptability is bounded by memory and compute constraints.
OptMap is intended to run alongside any ROS1 or ROS2 LiDAR odometry package, and open-source code (including a lightly modified open-source LiDAR odometry algorithm \cite{chen2023direct}) is released for easy development\footnote{\url{https://github.com/vectr-ucla/optmap}}.
OptMap could be adapted to other perception sensors given equivalent descriptors for mutual information.
To summarize, the contributions are:
\begin{itemize}
    \item \textbf{Continuous Exemplar-Based Clustering}. Novel submodular function which quantifies informativeness without bias from oversampling and improves evaluation time.
    \item \textbf{Dynamic Reordering}. Addresses sensitivity of streaming submodular algorithms to input order by approximating the marginal values of elements to ensure stream order prioritizes scans with high expected value.
    \item \textbf{Tight \emph{A Priori} Solution Bounds}. Save time via tight \emph{a priori} optimal solution value bounds derived from the simple geometric relationships between descriptors.
    \item \textbf{Ablation Studies}. Studies demonstrate improvements due to the novel reward function, dynamic reordering, and tight initial bounds.
    \item \textbf{Open-Source Code}. Released ROS1 and ROS2 packages. OptMap is meant to be run alongside any LiDAR odometry algorithm for quick testing and development.
\end{itemize}

%% file: sections/related_works.tex
\section{Related Works}
\label{sec:related_works}

\subsection{Geometric Maps for Mobile Robots} 
Key performance metrics for geometric maps include the level of detail and coverage area, the complexity of applying map updates (e.g., loop closures), and the level of flexibility or specificity a map has towards serving multiple functions \cite{slam-handbook}.
For LiDAR mapping, direct geometric maps are generated via the union of point clouds as complete, dense maps which use all collected scans \cite{cai2021ikd, xu2022fast} or as lightweight keyframe maps which contain a subset of selected scans \cite{zhang2014loam, chen2023direct, shan2020lio}.
Keyframe maps are easily maintained over long-duration missions as they require less memory and can be updated via pose-graph optimization, but complete maps provide a higher degree of detail.
3D occupancy grids are the most commonly used map type for autonomous planning algorithms \cite{hornung13auro}.
Their primary advantages are that they have minimal memory requirements, are easily searched, explicitly model free-space, and are easily generated from multiple fused sensors \cite{yue2018hierarchical, yue2019multilevel, garg2020semantics, cai2023occupancy}.
The downside is that they require computationally expensive ray tracing to generate, and can be slow to update in batches (e.g., after a loop closure).
Hierarchical graph-based maps can be memory efficient and carry significant detail \cite{kim2021plgrim}, but are often designed by experts for specific applications and therefore are not well-suited to multiple diverse algorithms.
A trend for recent online mapping algorithms is to reduce the memory burden of long-duration missions \cite{dong2025lidar, thorne2025submodular}.
This continued optimization means future maps generated from SLAM may become more sparse and thus less useful for downstream autonomy algorithms.

\textbf{Learned Point Cloud Global Descriptors} compress dense point clouds into compact vectors where the distance between descriptors encodes the similarity between point clouds.
While a few model-free LiDAR descriptors exist (e.g., \cite{cop2018delight}), most work in this field has been towards developing neural networks to generate descriptors for global place recognition.
Notable recent advances in learned point cloud descriptors include the use of transformers to generate yaw-invariant descriptors \cite{ma2022overlaptransformer, arce2023padloc}, and the use of 3D convolutions to enable descriptor generation on sparse point clouds \cite{komorowski2021minkloc3d}.
In \cite{chen2022overlapnet}, the norm of the difference between descriptors can be used to predict the number of points shared between two point clouds.

\subsection{Submodular Maximization}
Submodular maximization is a powerful tool for combinatorial optimization problems which exhibit monotonicity and diminishing returns because it provides efficient solutions to NP-hard problems with tight suboptimality guarantees.
The most important result in the field is the proven performance of the greedy algorithm when maximizing non-negative monotone submodular functions subject to cardinality constraints, which achieves a $1-\frac{1}{e} \approx 0.63$-suboptimality bound. 
This bound is optimal for any algorithm that queries the submodular function a polynomial number of times \cite{nemhauser1978analysis, nemhauser1978best}, where greedy requires $O(kN)$ queries (evaluations) to make a solution of size $k$ from an input set with $N$ elements.
The computation and memory complexity of the greedy algorithm has been improved \cite{minoux2005accelerated, mirzasoleiman2015lazier, mirzasoleiman2016fast}, with extensions proposed for non-monotone functions \cite{feige2011maximizing, gupta2010constrained, das2011submodular}.
Advances in submodular theory are frequently applied to computing and machine learning tasks such as dataset summarization, image classification, and search result generation \cite{dasgupta2013summarization, dueck2007non, el2011beyond}.
For example, Lin et al.~\cite{lin2011class} provides submodular functions that balance representation and diversity for document summarization. 
Fujishige et al.~\cite{fujishige2005submodular} provides coverage of classical algorithms, common submodular functions, and their constraints.

\textbf{Streaming Submodular} algorithms were developed to optimize over large input sets such as video streams or machine learning training datasets.
Two approaches to streaming submodular maximization were proposed concurrently by \cite{badanidiyuru2014streaming} and \cite{chakrabarti2015submodular}.
Both aimed to maximize a submodular function using only a single evaluation pass which queries the submodular function $O(N)$ times.
The fundamental idea of \cite{badanidiyuru2014streaming} is that provably near-optimal solutions can be built by identifying elements of the input set that are valuable relative to a known optimal solution value.
Given that the optimal value cannot be known \textit{a priori}, multiple solutions with unique guesses (separated by a granularity constant $\epsilon$) of the optimal value are maintained in parallel.
Bounds on the optimal solution value (known or found online) ensure that one of the maintained solutions has a guess that is close to the true optimal, providing a $(1/2-\epsilon)$-suboptimality bound.
One often overlooked property of \cite{badanidiyuru2014streaming} is that the suboptimality guarantee does not require a complete evaluation pass. 
For example, if all solutions reach the cardinality constraint before the full stream has been evaluated, the algorithm may terminate early and save a significant portion of the required computations.

A major limitation of streaming submodular maximization algorithms is the need to select solutions with limited information, i.e., before evaluating the entire stream.
Limited information degrades solution quality because low-quality elements may be included in solutions solely because they appear earlier in the stream.
Similarly, high-quality elements may be omitted because they are evaluated at the end of the stream.
Multiple methods to improve solution quality have been proposed such as \cite{chekuri2015streaming} which swaps solution elements when a higher quality element is presented.
Solution swaps are used in combination with randomized sub-sampling in \cite{feldman2018less} to achieve near-optimal solutions in expectation with fewer submodular function evaluations.
Several methods adopt a pseudo-streaming approach where batches of elements are considered together (e.g., using a greedy algorithm on subsets of the input set) \cite{kazemi2019submodular}, or where inflated solution sets are refined by post-processing \cite{alaluf2022optimal}.
These methods generally improve solution quality at the cost of either the deterministic nature of the algorithm or additional algorithmic steps which increase computation time. 
In specific cases where bias in input sets is known \emph{a priori} such as demographic problems, a partition constraint can be used to ensure the solution is not biased \cite{el2020fairness}.

Dynamic submodular maximization algorithms \cite{monemizadeh2020dynamic, banihashem2023dynamic, chen2022complexity, banihashem2024dynamic} have a similar motivation to streaming algorithms as they aim to maximize a submodular function through a series of input set insertions and deletions.
Dynamic algorithms are a generalization of streaming algorithms that allow for input set insertion and deletion and, therefore, suffer from worse memory and computation performance.

\textbf{Submodular Maximization in Robotics} largely centers on improving the performance and efficiency of autonomous decision making.
Many works prove that the underlying problem is submodular as theoretical justification for utilizing the greedy algorithm.
Graph optimization is readily formulated as submodular maximization, with notable works in pose graph sparsification \cite{chen2021anchor, kretzschmar2011efficient, khosoussi2019reliable}, path planning \cite{zhang2016submodular, bai2024multi}, and area search and coverage \cite{ramesh2024approximate, li2024computation}.
Multi-robot coordination has seen numerous advances in which submodularity can guaranty centralized and distributed planning with near-optimal coverage under tight communication constraints \cite{corah2019communication, corah2019distributed, gharesifard2017distributed}.
Submodularity is rarely found in perception algorithms, with notable exceptions in previously mentioned pose graph SLAM \cite{chen2021anchor, kretzschmar2011efficient, khosoussi2019reliable}, and some works in selecting either visual features \cite{carlone2018attention} or LiDAR scans \cite{thorne2025submodular} which best constrain localization.

%% file: sections/preliminaries.tex
\section{Preliminary on Submodularity}
\label{sec:preliminaries}

Let $X$ be a set of elements and $2^{X}$ be its power set.

\begin{definition}
    A set function $f:2^{X} \rightarrow \mathbb{R}_{\geq 0}$ is said to be \emph{monotonic non-decreasing} if $f(\emptyset) \geq 0$ and $f(B) > f(A)$ for all subsets $A \subseteq B \subseteq X$.
\end{definition}

\begin{definition}
\label{def:submodularity}
    A set function $f:2^{X} \rightarrow \mathbb{R}_{\geq 0}$ is \emph{submodular} if for $A \subseteq B \subseteq X$ and element $x \in X \, \backslash \, B$
\end{definition}
\vskip -0.25in
\begin{align}
    f(A \cup x) - f(A) \geq f(B \cup x) - f(B).
\end{align}

We refer to $f(A \cup x) - f(A)$ as the marginal value of $x$ with respect to $A$ and use the simplified notation $f(x|A)$.
The intuitive way to understand \cref{def:submodularity} with respect to submodular maximization is that the marginal value of any element must diminish as the solution increases in size.

Maximization of monotone non-decreasing, submodular functions subject to cardinality constraints is a well-studied problem, with the seminal result being the $(1-\frac{1}{e})$-suboptimality guarantee of the greedy algorithm \cite{nemhauser1978analysis, nemhauser1978best}.

\begin{definition}
\label{def:incomplete_information}
    ($\kappa$-wise information \cite{downie2022submodular}) Given a submodular function $f$, the $\kappa$-wise information set is defined as the set of tuples $\{(S, f(S)) | S \subseteq X, |S| \leq \kappa\}$. 
    When $\kappa = 2$, we refer to this as pairwise information.
\end{definition}

Knowing pairwise information for a given function $f$ and set $X$ is equivalent to having a look-up table for the value of $f$ given any solution $S$ such that $|S| \leq 2$.
Using pairwise information, the estimated marginal value of element $x \in X$ with respect to any solution $A \subseteq X$ can be defined as 
\begin{equation}
\label{eq:pairwise_information_approx}
    \overline{f}(x|A) = \min_{a \in A} f(x|a).
\end{equation}

The pairwise marginal value maintains submodularity and monotonicity of $f$, and \cite{downie2022submodular} proposes an optimistic greedy algorithm which is shown to produce empirically near-optimal results for several submodular maximization problems.
We use pairwise information in \cref{sec:dynamic_reordering} as a computationally lightweight means for approximating the marginal value of elements with respect to a solution when such marginal values would typically require expensive computations.

\subsection{Streaming Submodular Maximization}

\begin{algorithm}[t]
\caption{Sieve-Streaming Known Solution Bounds \cite{badanidiyuru2014streaming}}
\label{alg:seive-streaming}
\begin{algorithmic}[1]
    \renewcommand{\algorithmicrequire}{\textbf{Given:}}
    \renewcommand{\algorithmicensure}{\textbf{Output:}}
    \REQUIRE $E, \underline{OPT}, \overline{OPT}, k$ \\
    \ENSURE $\text{arg}\,\text{max}_{S_{v}} f(S_{v})$ \\
    
    \STATE {\tt\footnotesize\color{blue}// Create OPT guesses and solution sets}
    \STATE $O = \{(1+\epsilon)^{i} ~|~i \in \mathbb{Z},~ \underline{OPT} \leq (1+\epsilon)^{i} \leq \overline{OPT} \}$
    \FOR {$v \in O$}
        \STATE $S_{v} \leftarrow \emptyset$
    \ENDFOR
    
    \STATE {\tt\footnotesize\color{blue}// Make single pass over the input set}
    \FOR {$i \in \{1, 2, ..., |E|\}$}
        \STATE {\tt\footnotesize\color{blue}// Find marginal values on parallel threads}
        \FOR {$v\in O$}
            \IF {$f(e_{i} | S_{v}) \geq \frac{v/2-f(S_{v})}{k-|S_{v}|} \ \&\ |S_{v}| \leq k$}
                \STATE $S_{v} = S_{v} \cup e_{i}$
            \ENDIF
        \ENDFOR
    \ENDFOR
\end{algorithmic}
\end{algorithm}

Streaming submodular algorithms seek to maximize a submodular function using only a single pass over all elements in an input set.
If the optimal solution value $OPT$ were known \textit{a priori}, a solution could be iteratively built by comparing the marginal value of each element against $OPT$.
For example, if each element added to a solution with final size $k$ has a marginal value of at least $OPT/k$, then the produced solution would be optimal.
Using $OPT/k$ as a threshold is not practical in most cases, but the central contribution of \cite{badanidiyuru2014streaming} is to develop a threshold which yields a provably near-optimal solution.
The threshold proposed by \cite{badanidiyuru2014streaming} and used in \cref{alg:seive-streaming}, line 10 is

\begin{equation}
\label{eq:SS_thresh}
    T(k,S) = \frac{OPT/2-f(S)}{k-|S|}.
\end{equation}

In most cases, $OPT$ cannot be known \emph{a priori}, so multiple solutions are maintained where each solution has a guess, $v$, between \emph{a priori} optimal solution bounds $\overline{OPT} \geq v \geq \underline{OPT}$.
In the case where multiple solutions are updated throughout the stream, $OPT$ in \cref{eq:SS_thresh} is replaced with $v$.

\begin{proposition}
\label{prop:sieve-streaming}
    (Sieve-Streaming \cite{badanidiyuru2014streaming}) Given a submodular function f, known optimal solution bounds, and threshold $T(k,S)$, \cref{alg:seive-streaming} satisfies the following properties
    \begin{itemize}
        \item It outputs a solution $S$ such that $|S| \leq k$ and $f(S) \geq (\frac{1}{2}-\epsilon)*OPT$.
        \item It does 1 pass over the input set $E$, stores at most $O(\frac{k \log k}{\epsilon})$ elements and has $O(\frac{\log k}{\epsilon})$ update time per element.
    \end{itemize}
\end{proposition}

\begin{proof}
    \cref{prop:sieve-streaming} is a sub-case of Proposition 5.2 from \cite{badanidiyuru2014streaming} with \emph{a priori} $OPT$ bounds.
\end{proof}

Note that the for loop on lines 9-13 of \cref{alg:seive-streaming} evaluates the marginal value of each element with respect to all solutions.
This step can be conducted in parallel to dramatically improve computation times.
Also note that the suboptimality bound in \cref{prop:sieve-streaming} does not require the evaluation of all elements, i.e., if all solutions reach the cardinality constraint the algorithm may terminate early.

%% file: sections/Optimal_LiDAR_Map_Distillation.tex
\section{Optimal Geometric Map Distillation}
\label{sec:CEBC}

Geometric map distillation is a combinatorial optimization problem which can be formulated as
\begin{equation}
\label{eq:general_formulation}
    \max_{S \subseteq E, |S| \leq k} f(S), 
\end{equation}
where the input set $E$ represents the set of LiDAR scans collected in a mapping session.
In this work, we constrain solutions to be sets of scans in order to reduce the dimensionality of \cref{eq:general_formulation}.
The remainder of this section will define a function for geometric map \textit{informativeness}, i.e., $f(S)$ in \cref{eq:general_formulation}, which generates representative and diverse solutions when optimized.
``Representative'' and ``diverse'' are terms borrowed from \cite{lin2011class}, where representative solutions in the context of geometric map distillation are those that capture the most unique geometric information, and diverse solutions are those that capture as many different environments from the input set as possible.

The first step in formulating informativeness is to define mutual information, or the measure of redundant information between set elements for point clouds.
Using the defined mutual information metric, two submodular function classes for map distillation are defined: coverage functions, which minimize the mutual information between elements of the solution, and clustering functions, which maximize the mutual information between the solution and input sets.
We argue that clustering functions are better suited to geometric map distillation and develop a novel submodular clustering function called Continuous Exemplar-Based Clustering.
This function treats a mapping session as a continuous trajectory in order to correct unwanted bias towards repeated scans and save computation time.

\subsection{Point Cloud Mutual Information}
A natural metric for defining the mutual information of two point clouds $\mathcal{P}_{i}$ and $\mathcal{P}_{j}$ is overlap,  or the fraction of points in $\mathcal{P}_{i}$ that are sufficiently close to any point in $\mathcal{P}_{j}$.
Overlap is a useful metric for evaluating the similarity between two elements (i.e., scans), but finding overlap directly is  computationally expensive.
Instead, functions that measure (directly or indirectly) overlap are desired.

Learned descriptors can be trained to encode overlap using simple vector operations \cite{chen2022overlapnet}.
Descriptors in this work are unit-length $(n+1)-$dimension vectors that lie on the $n-$hypersphere denoted as $\mathbb{S}^{n}$.
We use $\phi(\mathcal{P}_{i})$ to denote the descriptor generated by point cloud $\mathcal{P}_{i}$.
The Euclidean distance between descriptors can be used to define point cloud similarity as
\begin{equation}
\label{eq:desc_mutual_information}
    \mathcal{I}(\mathcal{P}_{i}, \mathcal{P}_{j}) = ||\phi(\mathcal{P}_{i})-\phi(\mathcal{P}_{j})||_{2}.
\end{equation}
The descriptor neural network is trained such that two point clouds with more overlap have less distance between their descriptors.
Hence, we optimize geometric map distillation informativeness using descriptors and define the input set $E$ for \cref{eq:general_formulation} as a sequential (i.e., ordered) set of descriptors denoted as $E:=\{e_{1}, e_{2}, ..., e_{|E|}\}$ where $e_{i} = \phi(\mathcal{P}_{i})$ and $E \subseteq \mathbb{S}^{n}$.

\subsection{Submodular Map Coverage vs. Clustering}
\label{sec:coverage_vs_clustering}
We identify two possible function classes for geometric map distillation using learned descriptors as a measure of scan-to-scan mutual information.
We refer to these function classes as coverage and clustering using similar concepts from previous submodular summarization work \cite{lin2011class}. 
The goal of this subsection is to provide elementary coverage and clustering functions, prove their submodularity, describe the intuition and differences for each, and ultimately decide which class best captures the desired qualities of informativeness for geometric map distillation.


Coverage functions minimize mutual information between the elements of the solution set. 
To define a simple coverage function using descriptors, we formulate an area coverage problem over $\mathbb{S}^n$.
Consider each descriptor as the center of a hyperspherical cap with unit radius\footnote{A hyperspherical cap is the set of points on the surface of a hypersphere which are all within a fixed distance from the center of the cap. Interested readers can find formal definitions of spherical caps and their areas in} \cite{li2010concise}., where the cap with center $e$ is notated $Cap(e) \subset \mathbb{S}^n$.
We define an elementary coverage function as the area of the union of the caps
\begin{equation}
\label{eq:coverage_desc_simple}
    f_{cov}(S) = \mathcal{A}(\cup_{e_{i} \in S} Cap(e_{i})),
\end{equation}
where $\mathcal{A}(\cdot)$ is the surface area of a subset of $\mathbb{S}^n$.
Note that coverage encourages diverse solutions because high marginal value elements are those which are distant from other elements of the solution.
The primary advantage of coverage functions is that the evaluation complexity only scales with the size of the solution set which is often much smaller than the input set.
\Cref{eq:coverage_desc_simple} is proven monotone non-decreasing submodular in a similar manner to many other submodular area coverage functions (e.g., \cite{ramesh2024approximate} Proposition 2).

Clustering functions maximize the mutual information between the solution and input sets.
To define a simple clustering function using descriptors we use the k-mediod loss \cite{kaufman2009finding}
\begin{equation}
\label{eq:EBC_loss}
    L_{ebc}(S) = \frac{1}{|E|} \sum_{i=1}^{|E|} d(e_{i},S),
\end{equation}
where $d(e_{i},S)$ is the distance from $e_{i}$ to the nearest element of $S$.
K-mediod loss is used to define the exemplar-based clustering function 
\begin{equation}
\label{eq:EBC}
    f_{ebc}(S) = L_{ebc}(\mathbf{0}) - L_{ebc}(S\cup \mathbf{0}).
\end{equation}
Exemplar-based clustering is proven monotone non-decreasing, submodular in \cite{kaufman2009finding} given an appropriate null element $\mathbf{0}$.
The origin of the hypersphere $\mathbf{0} = [0,0,...,0] \in \mathbb{R}^{256}$ can be used in this case as it is a uniform distance of $1$ from all points on the sphere, so $L(\mathbf{0}) = 1$ and $f(\emptyset) = 0$.
Clustering functions identify both representative and diverse solution elements because an element must be both distinct from other solution elements and similar to elements of the input set in order to have a high marginal value.
However, this means evaluation complexity for clustering functions scales with the size of the input set and is generally more expensive than coverage functions.

\begin{figure}[t!]
    \centering
    \includegraphics[width=0.48\textwidth]{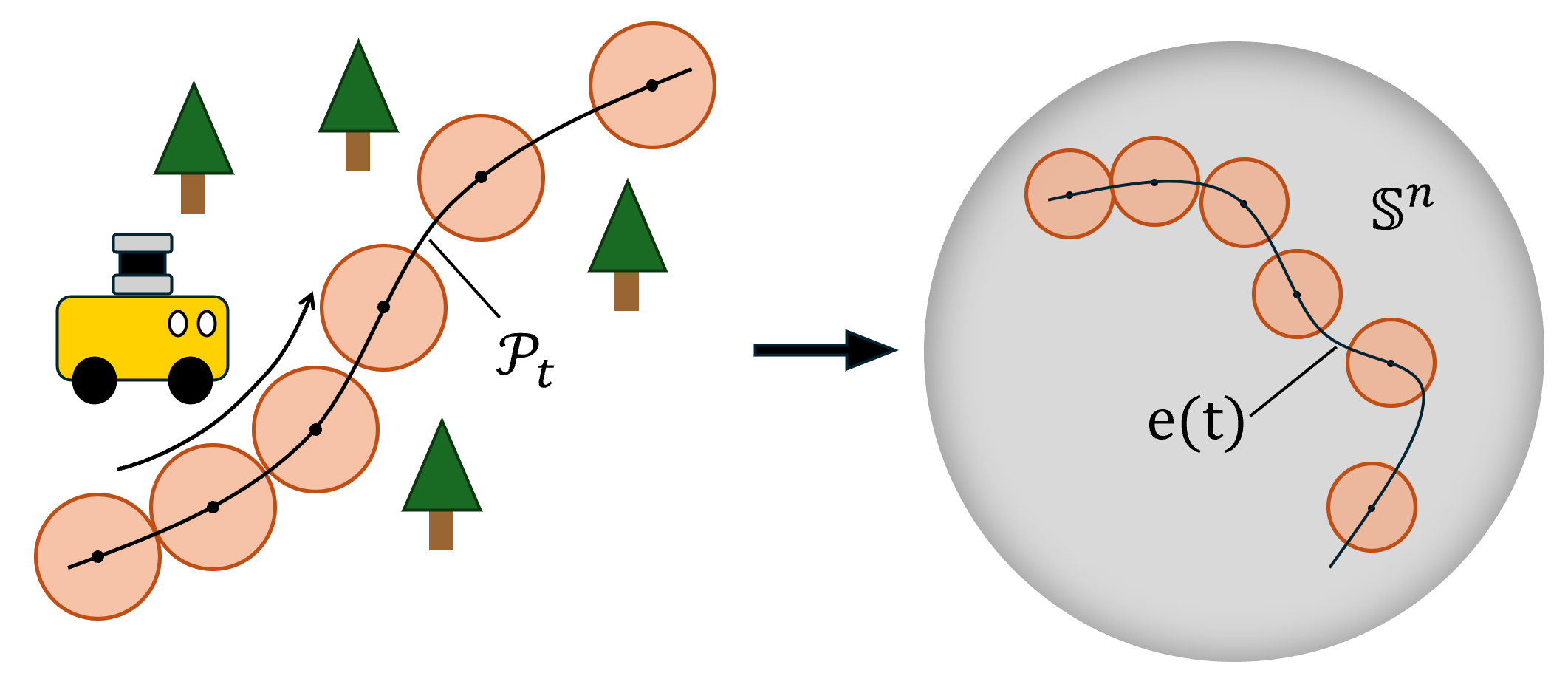}
    \caption{The descriptors of LiDAR point clouds collected along a time-parameterized trajectory are assumed to be samples along a corresponding trajectory in $\mathbb{S}^{n}$. Point clouds and therefore descriptors are assumed to be sampled with unequal distances.}
    \label{fig:assumption1}
    \vskip -0.15in
\end{figure}

Despite the expected difference in evaluation times, we argue that clustering functions are better suited for geometric map distillation because they reward both representativeness and diversity.
Although solution representativeness is difficult to directly quantify, it can be measured by the robustness of a function or how well it performs in the presence of unexpected outliers.
By rewarding both these properties, clustering functions can be expected to robustly select for high quality distilled maps (e.g., see \cref{subsec:CEBC_result}).
The primary disadvantage of clustering functions is that the evaluation time scales with the size of the input set.
However, we can derive a novel clustering function to mitigate this as described in the following subsection.

\subsection{Continuous Exemplar-Based Clustering}

Exemplar-based clustering \cref{eq:EBC} quantifies the average distance between all input set elements and the solution set.
However, LiDAR datasets often contain dense clusters of repeated scans (e.g., when the sensor is stationary or moving slowly), which biases exemplar-based clustering toward over-represented regions.
For example, if 50$\%$ of scans are captured before the sensor moves, half the clustering loss is determined by that single repeated scan.
However, it is possible to quantify the similarity between sequential scans in a mapping session and account for repeated scans via descriptors.

We propose a novel variant of exemplar-based clustering called Continuous Exemplar-Based Clustering (CEBC) which finds the unbiased clustering solution for a continuous mapping trajectory.
The intuition for CEBC is that geometric perception data is unique for a given pose which changes continuously in a trajectory, so LiDAR scans and descriptors can be treated as discrete measurements from a continuous trajectory.
Integrating along this trajectory should be a more accurate measure of the information captured in a mapping session, because sequentially similar samples will contribute less to the overall reward function value.
The notion that descriptors are discrete samples from a continuous trajectory is formalized by the following assumption.

\begin{figure}[t!]
    \centering
    \includegraphics[width=0.48\textwidth]{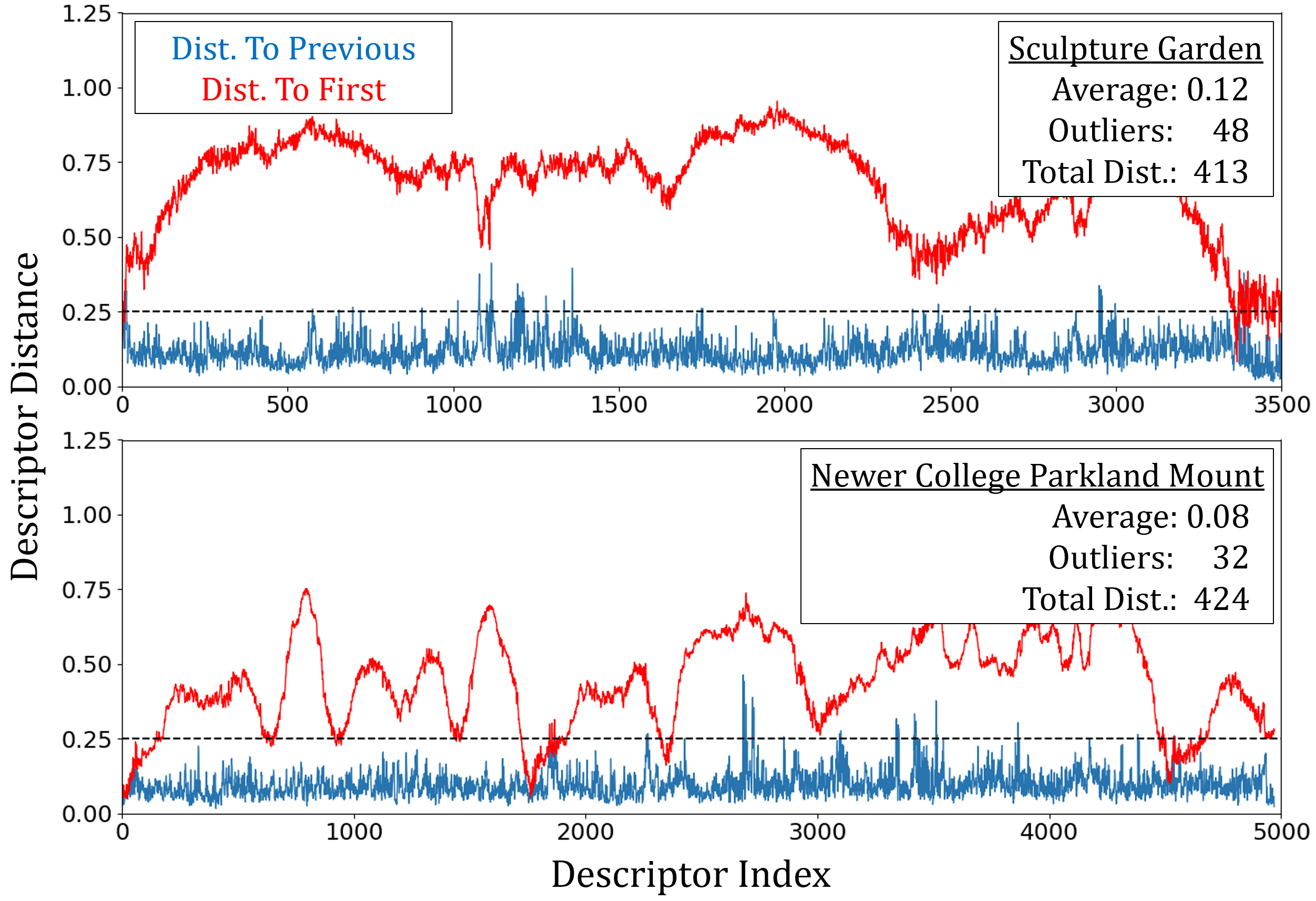}
    \caption{Descriptor distance between sequential descriptors in sculpture garden (top) and Newer College Dataset \cite{ramezani2020newer} Parkland Mount (bottom) datasets. Outliers is the number of sequential distances $>0.25$ and total distance is the sum of all sequential distances. The Euclidean distance to the first descriptor is shown in red as a reference.}
    \label{fig:assumption1-context}
    \vskip -0.2in
\end{figure}

\begin{assumption}
    Point cloud descriptors are assumed to be discrete samples from a \emph{continuous,} time-parameterized trajectory which we notate as $e(t): \mathbb{R} \rightarrow \mathbb{S}^{n} \, , \, t_{0} \leq t \leq t_{1}$.
    By continuity, we assume descriptors collected at a fixed frequency in the discrete input set are close to each other such that $||e_{i} - e_{i-1}||_{2} \leq \sigma$ for all $i$ where $2 \leq i \leq |E|$.
\label{assump:traj}
\end{assumption}

\Cref{assump:traj} comes from the fact that neural networks are smooth function approximators \cite{hornik1989multilayer}\cite[Chapter~6.6]{bengio2017deep} (\cref{fig:assumption1-context}).
In this setting, we assume collected LiDAR scans change gradually as the sensor moves through an environment.
With \Cref{assump:traj}, the CEBC loss term is then formulated as an integral over this trajectory
\begin{equation}
\label{eq:CEBC_loss}
    L_{cebc}(S) = \frac{1}{t_{1}-t_{0}}\int_{t_{0}}^{t_{1}} d(e(t),S) dt.
\end{equation}

\Cref{eq:CEBC_loss} can be transformed into a submodular function similarly to exemplar-based clustering using an appropriate null element.
This results in the formal definition of CEBC
\begin{equation}
\label{eq:CEBC}
    \Gamma(S) = L_{cebc}(\mathbf{0}) - L_{cebc}(S \cup \mathbf{0}),
\end{equation}
where $\mathbf{0} = [0,0,...,0] \in \mathbb{R}^{256}$ is the origin again.
Similar to \cref{eq:EBC}, because the origin is a unit distance from all points on $\mathbb{S}^{n}$, $L_{cebc}(\mathbf{0}) = 1$ and $\Gamma(\emptyset) = 0$.
\Cref{eq:CEBC} is hence our proposed optimization objective in \cref{eq:general_formulation} and still a set function as the solution is selected from the sampled descriptors.

\begin{proposition}
\label{prop:CEBC_monandsub}
    Let $\mathbb{S}^{n}$ be the set of all points on an n-hypersphere and $2^{\mathbb{S}^{n}}$ be its power set.
    The set function  $\Gamma: 2^{\mathbb{S}^{n}} \rightarrow \mathbb{R}$ from \cref{eq:CEBC} is a monotone non-decreasing, submodular function.
\end{proposition}
\begin{proof}
    We will prove that \cref{eq:CEBC} is monotone non-decreasing first, then prove submodularity.
    Recalling we use the notation $d(a,B) = \min_{b \in B} ||a - b||_{2}$, it is clear the loss term \cref{eq:CEBC_loss} is monotone non-increasing because of the min operator.
    Because the first term of \cref{eq:CEBC} is constant and the second term is the negative of a monotone non-increasing function, \cref{eq:CEBC} is therefore monotone non-decreasing.

    To prove submodularity via \cref{def:submodularity}, we must show that $\Gamma(x|A) \geq \Gamma(x|B)$ for sets $A \subseteq B \subseteq \mathbb{S}^{n}$ and $x \in \mathbb{S}^{n} / B$.
    By convention $\mathbf{0} \in A,B$ because it is appended to the second term in \cref{eq:CEBC}.
    Consider an arbitrary point $s \in \mathbb{S}^{n}$.
    We know that $d(s,A) \geq d(s,A \cup x)$ and $d(s,A) \geq d(s,B)$ because the minimum distance is non-increasing and $A$ is a subset of both $A \cup x$ and $B$.
    If $s$ is closer (or equidistant) to $B$ than $x$ then $d(s,B \cup x) = d(s,B)$.
    Then, using $d(s,A) \geq d(s,A \cup x)$ and $d(s,B \cup x) = d(s,B)$, one can show
    \begin{equation}
    \label{eq:prop2_step1}
        d(s,A) - d(s,A \cup x) \geq d(s,B) - d(s,B \cup x),
    \end{equation}
    where the right hand side of the inequality is just zero.
    
    If $s$ is instead closer to $x$ than any element in $B$, then $d(s,B \cup x) = d(s,\{x\})$.
    It also implies that $s$ is closer to $x$ than any element in $A$ and $d(s,A \cup x) = d(s,\{x\})$.
    Because $d(s,A) \geq d(s,B)$, subtracting $d(s,\{x\})$ from both sides and using $d(s, B \cup x) = d(s,A \cup x) = d(s,\{x\})$, again gives \cref{eq:prop2_step1}.
    Hence, in either case, \cref{eq:prop2_step1} must hold.
    Furthermore, since $s$ is arbitrary, we can integrate \cref{eq:prop2_step1} to get the integral inequality
    \begin{align*}
        \int_{t_{0}}^{t_{1}} & \Bigl( d(e(t),A) - d(e(t),A \cup x) \Bigr) dt \hfill \\
        & \geq \int_{t_{0}}^{t_{1}} \Bigl(d(e(t),B) - d(x(t),B \cup x) \Bigr) dt,
    \end{align*}
    By adding and subtracting $\int_{t_{0}}^{t_{1}}d(x(t),\mathbf{0})dt$ to both sides and dividing both sides by $t_{1} - t_{0}$, we can substitute \cref{eq:CEBC_loss} to get
    \begin{align*}
        L_{cebc}(\mathbf{0}) - L_{cebc}(A \cup e) &- L_{cebc}(\mathbf{0}) + L_{cebc}(A) \hfill\\
        \geq L_{cebc}(\mathbf{0}) - L_{cebc}(B &\cup e) - L_{cebc}(\mathbf{0}) + L_{cebc}(B),
    \end{align*}
    which yields $\Gamma(e|A) \geq \Gamma(e|B)$, as desired.
\end{proof}

Although \cref{eq:CEBC} theoretically addresses the problem of over-sampling in LiDAR datasets, the integral in \cref{eq:CEBC_loss} is impossible to evaluate because we do not have access to $e(t)$ as a continuous function.
Instead, \cref{eq:CEBC} uses a discrete approximation of the path integral in \cref{eq:CEBC_loss} when optimized for \cref{eq:general_formulation}
\begin{align}
\label{eq:CEBC_disc}
    L_{disc}(S)& =  
    \frac{1}{d_{tot}} \sum_{i=1}^{|E|} w_{i} d(e_{i},S),
\end{align}
where $L_{disc}$ indicates the discrete CEBC loss, $d_{tot}$ is the total path length and, $w_{i}$ is the weight for element $e_{i}$.
The discrete approximation of a path integral requires weighting each sample point by the path length to neighboring sample points\footnote{There are many methods for discretely approximating continuous integrals, here we rely on a simple first-order method.}.
We can approximate this element weighting using the distance between sequential elements so that $w_{i} = ||e_{i} - e_{i-1}||_{2}$.
The total path length is then approximated as $d_{tot} = \sum w_{i}$.

A property of \cref{eq:CEBC} is that sequentially identical elements (e.g., LiDAR scans taken from the same position) will have low weights and therefore reduce the bias towards over represented scans.
Because clustering function evaluation complexity scales with the size of the input set, we can also use this property to reduce computation time by carefully eliminating redundant scans.
This is formalized by constructing a \emph{reduced set} (denoted $R$ where $R \subseteq E$) which omits such elements. 
\Cref{alg:reduced_set_selection} describes the procedure for selecting this reduced input set.
Hence for evaluating \cref{eq:CEBC_disc}, $R$ and $E$ can be used interchangeably.
Therefore, the goal of the remainder of this section is to show that the value of CEBC does not change much when evaluated using a full or reduced set, and that any solution composed of elements from the full set has a close solution composed of elements from the reduced set.

\begin{figure}[t!]
    \centering
    \includegraphics[width=0.45\textwidth]{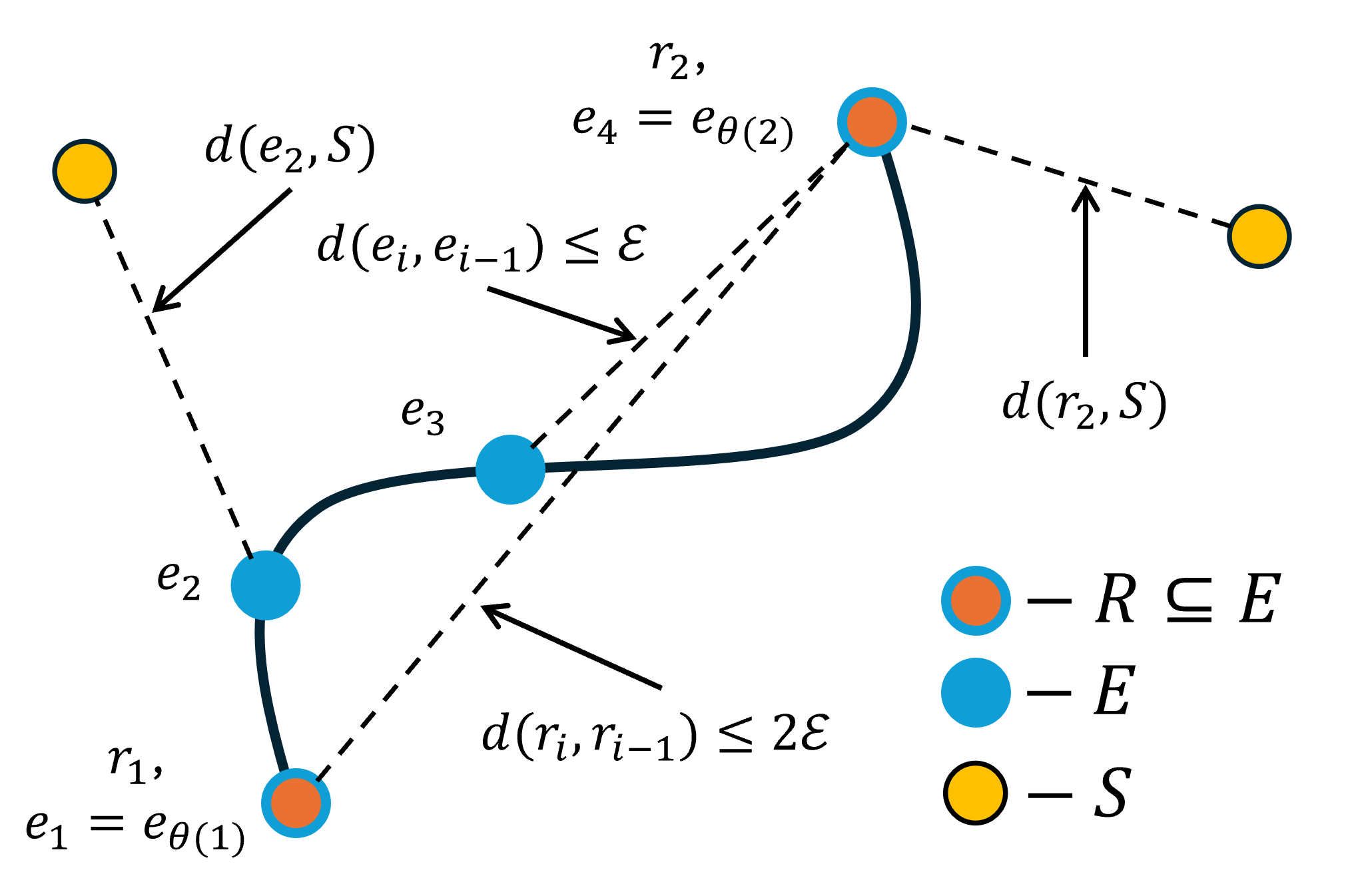}
    \caption{Section of a full and reduced set generated by \cref{alg:reduced_set_selection}. All elements in the full set (blue outlines) are no further than $\mathcal{E}$ apart, and all elements of the reduced set (orange interior) are no further than $2\mathcal{E}$ apart. The weight assigned to $r_{2}$ is the sum of weights from its source set $\hat{w}_{2} = w_{2}+w_{3}+w_{4}$. The closest element of $S$ to $e_{2}$ and $r_{2}$ does not have to be the same element of the solution $S$.}
    \label{fig:reduced_set_visual}
    \vskip -0.0in
\end{figure}

\begin{algorithm}[t]
\caption{Reduced Input Set Selection}
\label{alg:reduced_set_selection}
\begin{algorithmic}[1]
    \renewcommand{\algorithmicrequire}{\textbf{Given:}}
    \renewcommand{\algorithmicensure}{\textbf{Output:}}
    \REQUIRE $E, \mathcal{E} > \sigma$
    \ENSURE $R$
    \STATE $R = \{e_{1}\}, \hat{W} = \{0\}, \, w = 0$
    \FOR {$i=\{2,3,...,|E|\}$}
        \STATE $w = w + ||e_{i}-e_{i-1}||_{2}$
        \IF {$w \geq \mathcal{E}$}
            \STATE $R = R \cup \{e_{i}\}$
            \STATE $\hat{W} = \hat{W} \cup \{w\}$
            \STATE $w = 0$
        \ENDIF
    \ENDFOR
    \STATE $R = R \cup \{e_{|E|}\}$
    \STATE $\hat{W} = \hat{W} \cup \{w\}$
\end{algorithmic}
\end{algorithm}

\cref{alg:reduced_set_selection} gives our method for selecting a reduced input set with individual elements $r_{i} \in R$.
The elements of the reduced set are selected such that they all have at least $\mathcal{E}$ distance between sequential elements.
We assume that when building a reduced input set, an injective mapping function $\theta: \mathbb{Z} \rightarrow \mathbb{Z}$ is saved which translates indices from the reduced set to the corresponding index in the full set such that $r_{i} = e_{\theta(i)}$.
It will thus be useful to formalize the relationship between reduced set elements and the full set elements which were omitted.
We call the set of elements $\{e_{j} | j \in [\theta(i-1)+1, \theta(i-1)+2, ..., \theta(i)]\}$ the source set for $r_{i}$ and denote the set of indices for the $i$'th source set as $\eta_{i} = [\theta(i-1)+1, \theta(i-1)+2, ..., \theta(i)]$.
Elements in $R$ are assigned a weighting (weights for the reduced set are denoted as $\hat{w}$) for the purpose of evaluating \cref{eq:CEBC_disc} which is equal to the sum of weights in their source set, i.e., $\hat{w}_{i} = \sum_{j\in \eta_{i}}w_{j}$.
Intuitively, an element in the reduced set combines the value of all the elements of its source set.
In the following lemma, we prove that any reduced set element and its source set are approximately the same distance to any solution.
This follows under the assumption that all input set elements are close by continuity, i.e., $||e_{i} - e_{i-1}||_{2} < \sigma$.

\begin{lemma}
\label{lem:triangle}
    Consider an input set of elements $E$ and corresponding reduced set $R$ built using \cref{alg:reduced_set_selection} with $\mathcal{E} \geq \sigma$.
    For any element $r_{i}$ and any element from the corresponding source set $e_{j}, j \in \eta_{i}$, then $|d(r_{i},S)-d(e_{j},S)| \leq 2\mathcal{E}$.
\end{lemma}

\begin{proof}
    The distance between $r_{i}$ and its source set is bounded by $d(r_{i},e_{j}) \leq 2\mathcal{E}$, $\forall j \in \eta_{i}$ because in the worst case $d(r_{0},e_{\theta(0)+1}) = 0$, $d(e_{\theta(0)+1},e_{\theta(0)+2}) = \mathcal{E} - \delta$, and $d(r_{1},e_{\theta(0)+1}) = \mathcal{E} - \delta$ where $0 < \delta << 1$.
    Consider the elements in $S$ closest to $r_{i}$ and $e_{j}$ and denote them as $v$ and $u$ respectively.
    Note that $v$ and $u$ do not need to be the same element as shown in \cref{fig:reduced_set_visual}.
    Then the distance from $r_{i}$ to $S$ can be related to $e_{j}$ by $d(r_{i},v) \leq d(r_{i},u) \leq d(r_{i},e_{j})+d(e_{j},u) \leq 2\mathcal{E} + d(e_{j},u)$.
    This simplifies to $d(r_{i},v) \leq d(e_{j},u) + 2\mathcal{E}$, which can be manipulated into $|d(r_{i},S)-d(e_{j},S)| \leq 2\mathcal{E}$.
\end{proof}
\vskip -0.02in

In the following theorem, we prove that the value of a solution is arbitrarily similar for \cref{eq:CEBC} when using $E$ and $R$.

\begin{theorem}
\label{theorem:similarity_of_scores}
    Consider an input set of elements $E$ and corresponding reduced set $R$ built using \cref{alg:reduced_set_selection} with $\mathcal{E} \geq \sigma$.
    Let the submodular function $\Gamma_{E}(S) = L_{disc}(\mathbf{0}) - L_{disc}(S \cup \mathbf{0})$ indicate the input set (either $E$ or $R$) by the subscript.
    Then, for any solution $S$, we have $|\Gamma_{E}(S) - \Gamma_{R}(S)| \leq 2\mathcal{E}$.
\end{theorem}

\begin{proof}
    The proof follows from \cref{lem:triangle} as the value of each source set can be associated with the corresponding element of the reduced set. 
    Using \cref{eq:CEBC_disc}, we have
    \begin{equation*}
        |\Gamma_{E}(S) - \Gamma_{R}(S)|  = \frac{1}{d_{tot}} \Bigl| \sum_{i=1}^{|E|}w_{i} d(e_{i},S) - \sum_{i=1}^{|R|} \hat{w}_{i} d(r_{i},S) \Bigr|. 
    \end{equation*}
    Using $\sum_{i=1}^{|E|}(\cdot) = \sum_{i=1}^{|R|} \sum_{j \in \eta_i} (\cdot)$, 
    \begin{equation*}
        |\Gamma_{E}(S) - \Gamma_{R}(S)|  = \frac{1}{d_{tot}} \Bigl|\sum_{i=1}^{|R|} \sum_{j \in \eta_{i}} w_{j} d(e_{j},S) - \hat{w}_{i} d(r_{i},S)\Bigr| 
    \end{equation*}
    Factoring out $\hat{w}_{i}=\sum_{j \in \eta_{i}} w_{j}$,
    \begin{equation*}
        |\Gamma_{E}(S) - \Gamma_{R}(S)|  = \frac{1}{d_{tot}} \Bigl|\sum_{i=1}^{|R|}\sum_{j \in \eta_{i}}w_{j}(d(e_{j},S) - d(r_{i},S))\Bigr|. 
    \end{equation*}
    Finally, applying \cref{lem:triangle}, we get the desired inequality
    \begin{align*}
        |\Gamma_{E}(S) - \Gamma_{R}(S)| & \leq \frac{1}{d_{tot}} \Bigl|\sum_{i=1}^{|R|}\sum_{j \in \eta_{i}} 2 w_{j}\mathcal{E}\Bigr|  \\ 
        & = \frac{2\mathcal{E}}{d_{tot}} \sum_{i=1}^{|R|}\hat{w}_{i} \\ 
        & = 2\mathcal{E}. \qedhere
    \end{align*}
\end{proof}

\begin{figure*}
    \centering
    \vspace{6pt}
    \includegraphics[width=0.96\textwidth]{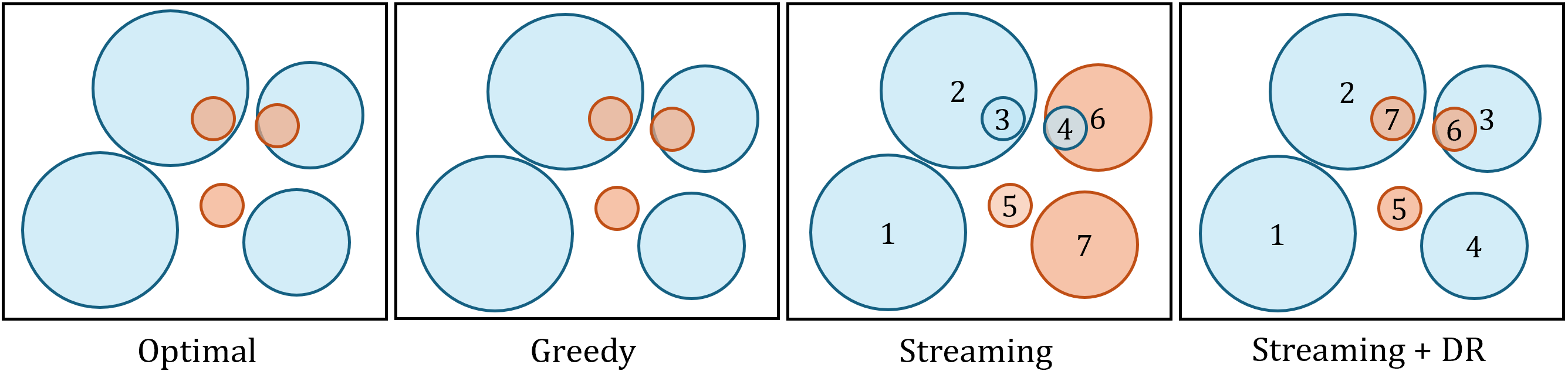}
    \caption{Sensor selection with a four sensor limit. Solution set in blue, other sensors in orange. From left to right: optimal solution, greedy solution which does not consider input order, streaming solution with arbitrary input order (indicated by numbers in sensor footprints), and the streaming solution with dynamic reordering (DR) which modifies the input order according to approximate element values. Sensors 3 and 4 selected in the streaming solution due to threshold saturation and opportunity cost bias respectively.} Assumed that $OPT$ is known \textit{a priori} and only one solution is maintained by the streaming variants.
    \label{fig:DR_cartoon}
    \vskip -0.2in
\end{figure*}

\Cref{theorem:similarity_of_scores} assumes no outliers (i.e., $||e_{i}-e_{i-1}||_{2} > \sigma$).
The following corollary demonstrates that the bound remains tight even in the presence of a small number of outliers.

\begin{corollary}
\label{cor:outliers}
    Given $N$ outlier descriptors where $||e_{i}-e_{i-1}||_{2} > \mathcal{E}$, for any $S$, $|\Gamma_{E}(S) - \Gamma_{R}(S)| \leq 2\mathcal{E} + \frac{2\mathcal{E}+4N}{d_{tot}}$.
\end{corollary}
\vskip -0.2in
\begin{proof}
    See \nameref{sec:appendix}.
\end{proof}

\cref{theorem:similarity_of_scores} and \cref{cor:outliers} establish that for any arbitrary solution, the value of \cref{eq:CEBC} evaluated on the full or reduced set is similar.
Next, we show that any solution selected from $E$ has a similar valued solution selected from $R$.

\begin{theorem}
\label{theorem:sim_sols}
    For any solution $S_{R} \subseteq R$ there exists a similar solution $S_{E} \subseteq E$ such that $\Gamma(S_{R}) \geq \Gamma(S_{E}) - \mathcal{E}$.
\end{theorem}

\begin{proof}
    By \cref{alg:reduced_set_selection}, any element of the full set has a corresponding element of the reduced set no further than $\mathcal{E}$ away.
    This means the elements of solution $S_{R}$ can each be shifted by up to $\mathcal{E}$ in order to find a corresponding solution $S_{E}$.
    Outliers do not need to be shifted because they are always included in both the full and reduced sets if they are further than $\mathcal{E}$ from the previous element.
    By optimally shifting each element of $S_{R}$ to $S_{E}$, each element of $E$ is now up to $\mathcal{E}$ closer to a solution element, i.e., $d(e_{i},S_{R} \cup \mathbf{0}) \leq d(e_{i},S_{E} \cup \mathbf{0}) + \mathcal{E}$.
    This results in a tight bound for the shifted solution value which is shown by using \cref{eq:CEBC} with the discrete loss \cref{eq:CEBC_disc}
    \begin{align*}
        \Gamma(S_{R}) =& L_{disc}(\mathbf{0}) - L_{disc}(S_{R} \cup \mathbf{0})\\
        \geq& 1-\frac{1}{d_{tot}}\sum_{i=1}^{|E|}w_{i}(d(e_{i},S_{E}\cup \mathbf{0})+\mathcal{E}).
    \end{align*}
    Separating the sum and using $\sum_{i=1}^{|E|}w_{i} = d_{tot}$,
    \begin{align*}
       \Gamma(S_{R}) = 1-\frac{1}{d_{tot}}\sum_{i=1}^{|E|}w_{i}(d(e_{i},S_{E}\cup \mathbf{0})) - \mathcal{E}.
    \end{align*}
    Simplifying yields the inequality $\Gamma(S_{R}) \geq \Gamma(S_{E})-\mathcal{E}$.
\end{proof}

Intuitively, \cref{theorem:sim_sols} says the average distance from an input set to solution set cannot be increased by more than $\mathcal{E}$ and thus any solution composed of elements from $E$ has an arbitrarily similar solution composed of elements from $R$.
To summarize, CEBC is well suited as a reward function for geometric map distillation because it selects the solution which best represents the input set, and avoids bias towards over-sampled positions.
Importantly, CEBC is submodular, meaning efficient algorithms for maximizing it with provably tight suboptimality guarantees exist.

%% file: sections/Dynamically_Reordered_Streaming_Submodular_Maximization.tex
\section{Dynamically Reordered Streaming Submodular Maximization}
\label{sec:dynamic_reordering}

Optimizing CEBC \cref{eq:CEBC} for geometric map distillation is computationally challenging because input sets can contain tens of thousands of scans.
However, \cref{eq:CEBC} is submodular, so multiple algorithms with polynomial-complexity are available to generate provably near-optimal solutions.
Streaming submodular maximization algorithms are well-suited for optimizing over large input datasets because they only require a single pass over the input set.
This is in contrast to the more popular greedy algorithm which requires $k$ passes to generate a solution of size $k$.
The downside of streaming is that elements are evaluated before the entire input set has been processed, so the algorithm must build solutions with limited information.
One such consequence of limited information is input order bias, or the increased likelihood of selecting elements due to the order of the stream.

This section presents a streaming submodular algorithm for geometric map distillation which efficiently addresses input order bias.
This is accomplished by using approximate marginal values of unchecked elements to dynamically reorder the stream such that the next element always has high expected value.
We begin by defining two forms of input order bias which degrade streaming algorithm performance relative to the greedy algorithm.
The dynamically reordered streaming submodular algorithm is then presented under the assumption that a suitable marginal value approximation function for CEBC exists.
We conclude the section by defining two functions for approximating the marginal value of \cref{eq:CEBC} which require $O(1)$ computations to evaluate.

\subsection{Input Order Bias in Streaming Submodular Maximization}

Input order bias is defined for streaming submodular maximization as the increased likelihood of solutions containing elements solely due to their position in the stream of the input set.
Bias is introduced to streaming algorithms because of the need to make decisions to include elements before evaluating the entire stream, i.e., using limited information. 
We identify two forms of input order bias which we refer to as opportunity cost and threshold saturation bias.

\textbf{Opportunity Cost Bias:} A naive approach to streaming submodular maximization would be to only admit elements with a marginal value of at least $OPT/k$.
If this approach managed to admit $k$ elements it is easy to show that the solution must be optimal, but \cite{badanidiyuru2014streaming} proves that such an approach fails for trivial scenarios.
Streaming algorithms must therefore rely on relaxed thresholds that admit suboptimal elements (e.g., \cref{eq:SS_thresh}) to achieve a near $1/2$-suboptimality bound.
This means that given two elements with similar value, the one seen earlier is more likely to be included in the generated solution because the solution might be full by the time it reaches the second element, or mutual information might lower the marginal value of the second element.

Opportunity cost bias is thus defined as the increased likelihood of streaming submodular algorithms to select elements that appear earlier in the stream \emph{due to relaxed admission thresholds}.
The loss of efficiency for streaming submodular maximization relative to the greedy algorithm due to opportunity cost bias is demonstrated in \cref{fig:DR_cartoon}.
In the third pane, the default streaming algorithm adopts an arbitrary stream order indicated by the numbers in the sensor footprints.
Although the 6th sensor in the stream covers more area than the 4th, it is not adopted in part because the 4th was considered earlier.

\textbf{Threshold Saturation Bias:} The relaxed admission threshold introduces another form of input order bias as it adjusts to the value of the current solution.
Recall that streaming submodular algorithms maintain multiple solutions in parallel, each with their own guess of the optimal solution value $v$.
When a solution value exceeds $v/2$, the admission threshold given by \cref{eq:SS_thresh} becomes negative, meaning that any evaluated element will be added to the solution after this point until the solution reaches the cardinality constraint.

Threshold saturation bias is thus the increased likelihood of selecting arbitrary or lesser elements as the admission threshold reaches $0$.
This effect is most pronounced when the threshold reaches $0$ early in the stream and with relatively few elements because a significant portion of the solution will be selected via random chance.
Threshold saturation bias is demonstrated in \cref{fig:DR_cartoon} where the threshold reaches $0$ after selecting the first two elements.
Once saturated, the remainder of the solution is determined solely by stream order and the 3rd sensor is chosen despite having zero marginal value.

\subsection{Streaming Submodular Dynamic Reordering}

We propose \emph{dynamic reordering} as our solution to input order bias in streaming submodular maximization for geometric map distillation.
The goal of dynamic reordering is to actively select stream order such that elements with high expected value are presented earlier and are therefore more likely to be included in solutions as a result of input order bias.
Expected marginal values, namely those obtained from an approximation of CEBC or heuristic denoted with $\hat{\Gamma}$, must be used because complete knowledge of the marginal value of each element would require evaluating the entire input set at least $k$ times, reducing it to a classical greedy approach.
Because each element may have a different marginal value with respect to each solution, \emph{ordering scores} are generated by taking the average expected marginal value for an element across all solutions.
Deciding which element to present next at each step in the stream introduces additional computational burden, but this can be mitigated by limiting the set of elements which can be searched over at each step.
The remainder of this section describes the process of dynamic reordering in greater detail.

\begin{algorithm}[t]
\caption{Dynamically Reordered (DR) Streaming Submodular Maximization}
\label{alg:dynamic-reordering}
\begin{algorithmic}[1]
    \renewcommand{\algorithmicrequire}{\textbf{Given:}}
    \renewcommand{\algorithmicensure}{\textbf{Output:}}
    \REQUIRE $E, \underline{OPT}, \overline{OPT}, k, F$ \\
    \ENSURE $\text{arg}\,\text{max}_{S_{v}} f(S_{v})$ \\

    \STATE {\tt\footnotesize\color{blue}// Create OPT guesses and solution sets}
    \STATE $O = \{(1+\epsilon)^{i} | i \in \mathbb{Z}, \underline{OPT} \leq (1+\epsilon)^{i} \leq \overline{OPT}\}$
    \FOR{$v \in \{1,2,...,|O|\}$}
        \STATE $S_{v} = \emptyset$
    \ENDFOR

    \STATE {\tt\footnotesize\color{blue}// Create partitions of indices and order scores}
    \STATE $\Theta_{S} \leftarrow ones(F*k)$
    \STATE $\Theta_{U} \leftarrow ones(|E|-F*k)$

    \STATE {\tt\footnotesize\color{blue}// Make single pass over the input set}
    \FOR{$a \in \{1,2,...,|E|\}$}
        \STATE $i = NextIndex(\Theta_{S}, \Theta_{U})$
        \STATE {\tt\footnotesize\color{blue}// Parallel evaluation step over all solutions}
        \FOR{$v \in \{1, 2, ..., |O|\}$}
            \STATE $ElementEvaluation(i, E, \Theta_{S}, \Theta_{U}, S_{v})$
        \ENDFOR
    \ENDFOR
\end{algorithmic}
\end{algorithm}

\begin{algorithm}[t]
\caption{Next Index}
\label{alg:nextind}
\begin{algorithmic}[1]
    \renewcommand{\algorithmicrequire}{\textbf{Given:}}
    \renewcommand{\algorithmicensure}{\textbf{Output:}}
    \REQUIRE $\Theta_{S}, \Theta_{U}$ \\
    \ENSURE $i$

    \STATE $i = \arg\max_{j}\Theta_{S}(e_{j})$
    \STATE $\Theta_{S}.remove(e_{i})$
    \STATE $\Theta_{S}.append(\Theta_{U}.first)$
    \STATE $\Theta_{U}.removeFirst()$
\end{algorithmic}
\end{algorithm}

\begin{algorithm}[t]
\caption{Element Evaluation}
\label{alg:DR_step}
\begin{algorithmic}[1]
    \renewcommand{\algorithmicrequire}{\textbf{Given:}}
    \renewcommand{\algorithmicensure}{\textbf{Output:}}
    \REQUIRE $i, E, \Theta_{S}, \Theta_{U}, S_{v}$ \\

    \STATE $\Theta_{S,n} = \Theta_{S}, \Theta_{U,n} = \Theta_{U}$
    \STATE $\Gamma(e_{i}|S_{v}) = 0$
    \FOR {$j \in \{1, 2, ..., |E|\}$}
        \STATE $d = ||e_{i} - e_{j}||_{2}$
        \IF {$d < d(e_{j},S_{v})$}
            \STATE $\Gamma(e_{i}|S_{v}) += \frac{w_{j}}{d_{tot}}(d(e_{j},S_{v}) - d)$
            \IF {$e_{j} \in \Theta_{s}$}
                \STATE $\Theta_{S,n}(e_{j}) += \frac{1}{|O|}(\hat{\Gamma}(d) - \hat{\Gamma}(d(e_{j},S_{v})))$
            \ELSIF {$e_{j} \in \Theta_{u}$}
                \STATE $\Theta_{U,n}(e_{j}) += \frac{1}{|O|}(\hat{\Gamma}(d) - \hat{\Gamma}(d(e_{j},S_{v})))$
            \ENDIF
        \ENDIF
    \ENDFOR

    \IF {$\Gamma(e_{i}|S_{v}) \geq \frac{v/2 - f(S_{V})}{k-|S_{v}|} \& |S_{v}| \leq k$}
        \STATE $S_{v} = S_{v} \cup e_{i}$
        \STATE $\Theta_{S} = \Theta_{S,n}, \Theta_{U} = \Theta_{U,n}$
    \ENDIF
\end{algorithmic}
\end{algorithm}

The dynamically reordered streaming submodular algorithm for geometric map distillation is given by \cref{alg:dynamic-reordering}.
Intuitively, \cref{alg:dynamic-reordering} closely mirrors Sieve-Streaming (\cref{alg:seive-streaming}) with additional steps for efficiently tracking and updating the score of CEBC and the ordering scores used to select stream order.
Multiple solutions are initialized on lines 2-5 identically to Sieve-Streaming.
Lines 7-8 include the first deviation, where a sorted and unsorted set are initialized.
Ordering scores are kept in two sets, the sorted set $\Theta_{s}$ and unsorted set $\Theta_{u}$.
The ordering score for element $e_{i}$ is denoted as $\Theta_{s}(e_{i})$ or $\Theta_{u}(e_{i})$ depending on which set it belongs to.
We note that the input set order is randomized before optimization.
The single evaluation pass occurs on lines 10-16, where details of key functions are left to their own algorithm blocks.

\Cref{alg:nextind} returns the index of the highest scored element from the sorted set.
The purpose of the sorted and unsorted sets is to minimize the burden of sorting the complete list of ordering scores.
Finding the highest order score in the full input set requires a sort operation with complexity $O(NlogN)$ for sets with $N$ elements.
This added complexity is mitigated by only sorting a set with fixed size $Fk$ where $F$ is a parameter and $k$ is the cardinality constraint.

\Cref{alg:DR_step} describes the steps for evaluating elements and updating ordering scores.
It begins by making a copy of the sorted and unsorted sets (line 1) which will be saved if the element is included in the solution.
The marginal value $\Gamma(e_{i}|S_{v})$ of the element is found incrementally during the evaluation pass (lines 3-13).
The marginal value of $e_{i}$ is determined by finding the elements of the input set $e_{j} \in E$ such that $d(e_{i},e_{j}) < d(e_{j},S)$.
If an input set element is closer to $e_{i}$ than the current solution (lines 4-5), then the marginal value score is increased by the difference in the distances weighted by the input set element's weighting.
On lines 7-11, the ordering score is updated by taking the difference in $\hat{\Gamma}$ divided by the number of solutions such that the average value is maintained (lines 8 and 10).
The key idea of dynamic reordering is that the evaluation of CEBC and ordering scores is performed at the same time.
$\hat{\Gamma}$ is assumed monotone-increasing such that when $d < d(e_{i},S_{v})$, $\hat{\Gamma}(d) < \hat{\Gamma}(d(e_{i},S_{v}))$.
This means that lines 8 and 10 decrease the ordering score because this new element has a higher degree of mutual information with it.

\subsection{Approximating CEBC for Dynamic Reordering}
\label{sec:CEBC_approx}

Central to the efficiency of dynamic reordering is the ability to approximate the relative marginal value of an element with respect to a given solution, denoted as $\hat{\Gamma}(e_{i}|S_{v})$ in \cref{alg:DR_step} lines 8 and 10.
In this subsection, two such approximations with $O(1)$ evaluation complexity are derived.
The first uses scan poses to identify elements that were collected far from the current solution, and the second approximates \cref{eq:CEBC} as an area coverage problem and leverages the pairwise marginal value approximation from \cref{def:incomplete_information} to reduce evaluation complexity.
Since only relative ordering is required, the two approximations can also be combined into a third function that leverages the advantages of both.

\textbf{A position-based approximation} is formulated using the known positions of each element in the dataset.
Given the position of all elements in the solution set as $X_{S} = [\textbf{x}_{1}, \textbf{x}_{2}, ..., \textbf{x}_{|S|}]$ where $\textbf{x}_{i} \in \mathbb{R}^{3}$, the pose approximate marginal value of an element with position $\textbf{x}_{e}$ is
\begin{equation}
\begin{split}
\label{eq:pose_heur}
    \hat{\Gamma}(x) & = \Gamma_{pose}(d(\textbf{x}_{e},X_{S})) \\
    & = 1-\max \left(0, -\log \left( d(\textbf{x}_{e},X_{S})/\alpha + 0.1\right)\right),
\end{split}
\end{equation}
where $\alpha$ is a tunable radius parameter.
The log function introduces a decay which is intended to mimic the overlapping area of 2D sensors with circular footprints.
This heuristic favors spatially well-separated solutions, which is effective in environments where coverage correlates with position such as outdoor mapping. 
However, it can break down when small positional differences correspond to large changes in the environment (i.e., scans taken from either side of a wall or between building floors).

\textbf{A descriptor-coverage approximation} is formulated by converting \cref{eq:CEBC} into an area coverage function (similar to \cref{eq:coverage_desc_simple}) and simplifying this function using the pairwise marginal value approximation \cref{eq:pairwise_information_approx}.
For the reader's convenience, we recall that $\mathbb{S}^{n}$ is the hypersphere which all descriptors lie on. 
The surface area of a subset of the hypersphere is denoted by the function $\mathcal{A}: \mathbb{S}^{n} \rightarrow \mathbb{R}$ where the surface area for any $A \subseteq \mathbb{S}^{n}$ is defined as $\int_{A}1d\boldsymbol{\sigma}$.
Also recall that \cref{eq:coverage_desc_simple} is the surface area of a set of spherical caps centered at the solution elements (descriptors).
It is convenient to imagine $\mathbb{S}^{n}$ as a plane or familiar $2-$sphere, as much of the same intuition holds despite the high dimensionality.
To convert \cref{eq:CEBC} into a coverage function we assume the solution set is selected from $E$, but evaluated by integrating over $\mathbb{S}^{n}$.
By replacing $E$ with $\mathbb{S}^{n}$ in \cref{eq:CEBC_loss}, a coverage loss function is derived
\begin{align}
\begin{split}
\label{eq:coverage_step1}
    L_{cebc} \approx
    L_{cover}(S) = \frac{1}{\mathcal{A}(\mathbb{S}^{n})} \int_{\mathbb{S}^{n}} d(x,S) d\boldsymbol{\sigma}.
\end{split}
\end{align}
where $\mathcal{A}(\mathbb{S}^{n})$ is again the surface area of the hypersphere.
Converting this loss term into a submodular function is similar to exemplar-based clustering and CEBC
\begin{align}
\label{eq:coverage_step2}
    f_{cover}(S) = L_{cover}(\mathbf{0}) - L_{cover}(S \cup \mathbf{0}),
\end{align}
using the origin $\mathbf{0} = [0,0,...,0] \in \mathbb{R}^{n}$ as the null element.

Evaluating \cref{eq:coverage_step2} remains computationally expensive because it requires integrating over the surface of a high-dimension hypersphere.
Additionally, the fraction of the hypersphere covered by a spherical cap is vanishingly small \cite{li2010concise}, so \cref{eq:coverage_step1} will equal $1$ even for very large solutions.
Both these problems can be solved by separating the hypersphere into regions that are either close to or far from elements of a solution and integrating over these regions separately.
The region near $S$ is defined as the union of all unit radius spherical caps centered at its descriptors, notated as $\mathcal{C}(S) := \{x \in \mathbb{S}^{n}| \ ||e-x||_{2} \leq 1, \forall e \in S\}$.
The complement of $\mathcal{C}(S)$ is defined accordingly as $\mathcal{C}^{-}(S) := \{x \in \mathbb{S}^{n} | \ ||e-x||_{2} > 1, \forall e \in S\}$.
Because $\mathcal{C}$ and $\mathcal{C}^{-}$ are complements, $\mathcal{A}(\mathbb{S}^{n}) = \mathcal{A}(\mathcal{C})+\mathcal{A}(\mathcal{C}^{-})$.

An equivalent function for maximizing \cref{eq:coverage_step2} that only integrates over $\mathcal{C}$ is
\begin{equation}
\label{eq:coverage_step3}
    \Gamma_{C}(S) = \mathcal{A}(\mathcal{C}) - \int_{\mathcal{C}} d(x,S) d\boldsymbol{\sigma}.
\end{equation}

We show that \cref{eq:coverage_step2} achieves maximum value with the optimal solution of $\cref{eq:coverage_step3}$ in the following theorem.

\begin{theorem}
    The elements that maximize \cref{eq:coverage_step2} are equivalent to the elements that maximize \cref{eq:coverage_step3}.
\end{theorem}
\begin{proof}
    Note that $\mathbf{0}$ is always unit distance from any surface element of the hypersphere, so evaluating \cref{eq:coverage_step1} with only the null element is $L_{cover}(\mathbf{0}) = \frac{1}{\mathcal{A}(\mathbb{S}^{n})}\int_{\mathbb{S}^{n}}1d\boldsymbol{\sigma} = 1$.
    By plugging in \cref{eq:coverage_step1} and using $L_{cover}(\mathbf{0}) = 1$, \cref{eq:coverage_step2} can be rewritten as
    \begin{align*}
        f_{cover}(S) &= 1 - \frac{1}{\mathcal{A}(\mathbb{S}^{n})}\int_{\mathbb{S}^{n}} d(x,S \cup \mathbf{0}) d\boldsymbol{\sigma}\\
        &= \frac{1}{\mathcal{A}(\mathbb{S}^{n})} (\mathcal{A}(\mathbb{S}^{n}) - \int_{\mathbb{S}^{n}} d(x,S \cup \mathbf{0}) d\boldsymbol{\sigma}).
    \end{align*}
    Because $\mathcal{C}$ is the portion of $\mathbb{S}^{n}$ near to $S$, we can evaluate the integral as the sum of integrals over components $\mathcal{C}$ and $\mathcal{C}^{-}$ so that $\int_{\mathbb{S}^{n}} d(x,S \cup \mathbf{0}) d\boldsymbol{\sigma} = \int_{\mathcal{C}^{-}}1d\boldsymbol{\sigma} + \int_{\mathcal{C}}d(x,S)d\boldsymbol{\sigma}$,
    \begin{equation*}
        f_{cover}(S) = \frac{1}{\mathcal{A}(\mathbb{S}^{n})} \left(\mathcal{A}(\mathbb{S}^{n}) - \int_{\mathcal{C}^{-}}1d\boldsymbol{\sigma} - \int_{\mathcal{C}} d(x,S) d\boldsymbol{\sigma}\right).
    \end{equation*}
    Recalling $\int_{\mathcal{C}^{-}}1d\boldsymbol{\sigma} = \mathcal{A}(\mathcal{C}^{-})$ and $\mathcal{A}(\mathcal{C}) = \mathcal{A}(\mathbb{S}^{n}) - \mathcal{A}(\mathcal{C}^{-})$,
    \begin{equation*}
        f_{cover}(S) = \frac{1}{\mathcal{A}(\mathbb{S}^{n})}\left(\mathcal{A}(\mathcal{C}) - \int_{\mathcal{C}} d(x,S) d\boldsymbol{\sigma}\right),
    \end{equation*}
    which, using \cref{eq:coverage_step3}, simplifies to
    \begin{equation*}
        f_{cover}(S) = \frac{1}{\mathcal{A}(\mathbb{S}^{n})}\Gamma_C(S).
    \end{equation*}
    Taking $\argmax$ over $S$ of both sides yields
    \begin{equation*}
        \argmax_S f_{cover}(S) = \argmax_S \Gamma_C(S)
    \end{equation*}
    where we have used the property that scaling a cost function by a constant does not change the optimal solution. 
    Hence, the elements that maximize \cref{eq:coverage_step2} also maximize \cref{eq:coverage_step3}.
\end{proof}

\begin{figure}[t!]
    \centering
    \includegraphics[width=0.48\textwidth]{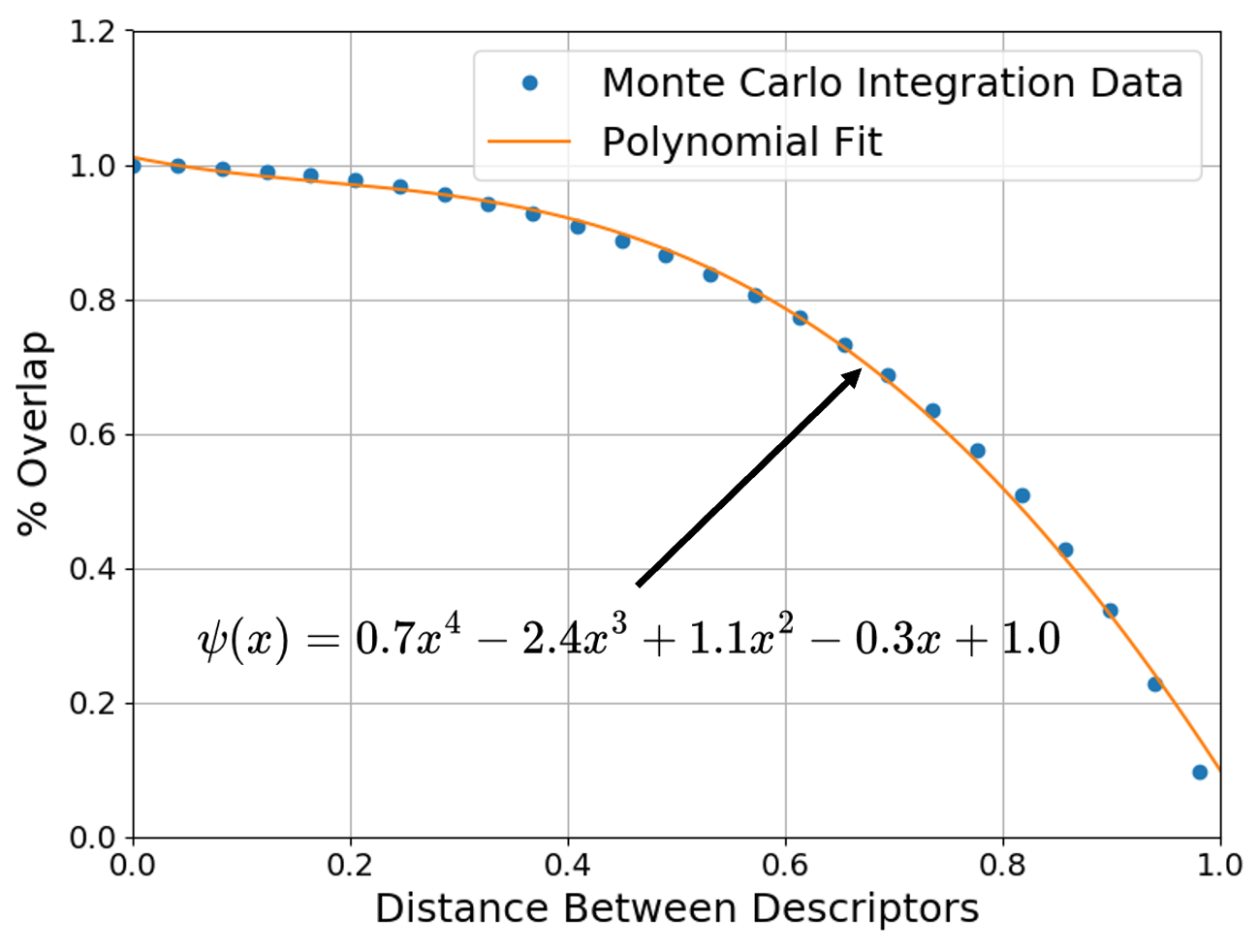}
    \caption{Mutual information with respect to \cref{eq:coverage_step3} as percent overlap between spherical caps. Overlap percent is obtained as a function of the distance between descriptors using a dense Monte Carlo integration over the surface of $\mathbb{S}^{n}$, with a fourth order polynomial function $\psi(x)$.}
    \label{fig:2-inf-overlap}
    \vskip -0.15in
\end{figure}

Computing the marginal value of \cref{eq:coverage_step3} remains difficult due to integrating over multiple spherical caps, but can be simplified using pairwise information.
Note that any solution with one element will have constant value as any two solutions could be related by an isometric rotation,
\begin{equation}
    \Gamma_{C}(\{e\}) = \mathcal{A}(\mathcal{C}) - \int_{\mathcal{C}}d(x,\{e\})d\boldsymbol{\sigma} = c.
\end{equation}

We define the degree of mutual information between two elements $e_{i}$ and $e_{j}$ as $\mathcal{O}(d(e_{i},\{e_{j}\})):\mathbb{R} \rightarrow \mathbb{R}$. 
Thus, the value of any solution with two elements is $\Gamma_{C}(S) = c + c (1-\mathcal{O}(d(e_{i},\{e_{j}\})))$.
$\mathcal{O}(d(e_{i},\{e_{j}\}))$ can be computed offline and approximated using the polynomial given as $\psi(x)$ in \cref{fig:2-inf-overlap}.
Spherical caps have no overlap if their centers are further than $1.1$ so the pairwise mutual information function is 
\begin{equation}
    \mathcal{O}(d(e_{i},\{e_{j}\})) = 
    \begin{cases}
        \psi(d(e_{i},\{e_{j}\})) &\text{if } d(e_{i},\{e_{j}\})<1.1\\
        0 &\text{otherwise}.
    \end{cases}
\end{equation}

The pairwise marginal value approximation of \cref{eq:coverage_step3} can then serve as an approximation for the marginal value of \cref{eq:CEBC}.
Specifically, the approximation function $\hat{\Gamma}(x)$ used for generating order scores in dynamic reordering is
\begin{equation}
\begin{split}
\label{eq:desc_heur}
    \hat{\Gamma}(x) = \overline{\Gamma}_{C}(d(e,S)) \propto 1-\mathcal{O}(d(e,S)).
\end{split}
\end{equation}
When $\hat{\Gamma}(x)$ is evaluated on lines 8 and 10 of \cref{alg:DR_step}, the new element is already known to be closer than any element in the solution so $d(e,S)$ is assumed known when evaluating \cref{eq:desc_heur}.
The single element score constant $c$ is dropped because only the relative value of ordering scores is required.

The value of any solution with two elements can be reduced to a single input function because their overlap can be defined using their pairwise distance.
Finding the value for solutions with three or more elements would require an overlap function with at least three inputs.
In fact, finding the total area covered by $p$ caps up to any isometric transformation requires $\frac{p(p-1)}{2}$ pairwise distances.
Thus in order to find an $O(1)$ approximation for solutions with more than two elements similar to \cref{eq:desc_heur} would require either a lookup table with $\frac{p(p-1)}{2}$ dimensions or finding coefficients for a multilinear polynomial.
It is not immediately obvious that either of these approximations would remain tractable or accurate and are therefore not used.

%% file: sections/Algorithmic_Implementation.tex
\section{Algorithmic Implementation}
\label{sec:alg_imp}

This section describes algorithmic implementation details which reduce OptMap's computation time and allow for practical application through solution constraints.
The first subsection describes our method for deriving tight \emph{a priori} solution bounds, and the second details simple position and time constraints used for the included online geometric change detection example.

\subsection{Initial Solution Bounds}
\label{subsec:bounds}

\begin{algorithm}[t]
\caption{Heuristic Solution for Initial Lower Bound}
\label{alg:heur_sol}
\begin{algorithmic}[1]
    \renewcommand{\algorithmicrequire}{\textbf{Given:}}
    \renewcommand{\algorithmicensure}{\textbf{Output:}}
    \REQUIRE $E, k$ \\
    \ENSURE $\underline{OPT} \leftarrow f(S_{h})$ \\
    \STATE $d_{tot} = 0$
    \FOR {$i \in \{2,3,...,|E|\}$}
        \STATE $d_{tot} = d_{tot} + ||e_{i} - e_{i-1}||_{2}$
    \ENDFOR
    \STATE $S_{h} = \emptyset, i=1$
    \WHILE {$|S_{h}| \leq k$}
        \STATE $d = 0$
        \WHILE {$d \leq d_{tot}/k$}
            \STATE $d = d + ||e_{i} - e_{i-1}||_{2}$
            \STATE $i = i+1$
        \ENDWHILE
        \STATE $S_{h} = S_{h} \cup e_{i}$
    \ENDWHILE
\end{algorithmic}
\end{algorithm}

The time complexity of streaming submodular maximization algorithms can vary significantly depending on the existence or quality of \emph{a priori} solution value bounds.
This is because tight bounds allow for fewer maintained solutions which ultimately require fewer computations.
\cref{alg:heur_sol} describes our method for generating a heuristic solution which can be used to obtain an initial lower bound.
The heuristic solution finds $k$ elements with equal distance \textit{along the descriptor trajectory}, where the equal spacing works best in mapping datasets which explore new environments at consistent rates.
Lines 1-4 are performed alongside the reduced set selection from \cref{alg:reduced_set_selection}, where the total descriptor trajectory length is calculated by summing the total distance between all sequential input set elements.
The heuristic solution is thus built (lines 5-12) by selecting elements from the input set at regular intervals of $d_{tot}/k$.
The lower bound is obtained by evaluating the heuristic solution on \cref{eq:CEBC} which requires one evaluation pass.

The maximum value of \cref{eq:CEBC} is simple to obtain and is achieved when the solution set contains the entire solution set.
In this case, the second loss term in \cref{eq:CEBC} equates to $0$, and thus the upper bound is always $\overline{OPT} = 1$.



\subsection{Position and Time Constraints}
\label{sec:pos_time_constr}

The input set can be further constrained by applying additional position and time constraints.
Position constraints are defined using a set of balls $(x_{c}, r)$, where $x_{c}$ is the coordinates of the ball center and $r$ the radius, and time constraints are defined with a set of initial and final timestamps $(t_{i}, t_{f})$.
Each element of the input set is a scan with an associated position $x(e_{i})$ and time $t(e_{i})$.
The filtered input set is then the subset of the input set with poses and timestamps inside the defined position and time constraints, defined as $E_{c} := \{e_{i} \in E | \, ||x(e_{i})-x_{c}||_{2} \leq r \, , \, t(e_{i}) \in [t_{i},t_{f}]\}$.
Multiple constraints can be applied, giving OptMap the ability to distill specified portions of a mapping session.
An example of the position constraints are shown in \cref{fig:top-right-changedetect}, where position and time constraints are seen as valuable tools for customizing distilled maps to specific downstream applications.
We note that these simple constraints are sufficient for the included online change detection example, but more complex constraints could be substituted for other novel applications.

%% file: sections/results.tex
\section{Results}
\label{sec:results}

In this section, we provide results that demonstrate OptMap's theoretical contributions, expected distilled map quality, and computational performance.
The first experiment identifies the best function for quantifying informativeness in geometric map distillation and examines the effects of the reduced set selection threshold.
The second result demonstrates that dynamic reordering effectively reduces input order bias.
The third result highlights the computation time benefits of using tight \emph{a priori} solution bounds.
The final result demonstrates a specific application of OptMap in real-time geometric change detection.
Datasets include custom LiDAR mapping sessions collected on the UCLA campus using 32-beam Ouster OS0 and OS1 sensors (Sculpture Garden, Courtyard, and Lab), and two custom datasets from the Army Research Laboratory collected using a 128-beam Ouster OS1 sensor.
Open-source mapping sessions from the Newer College \cite{ramezani2020newer}, Oxford Spires \cite{tao2025spires}, and Semantic KITTI \cite{behley2019iccv} datasets are also used for evaluation.

Dense comparison maps are generated by combining 250 scans then applying a 0.5 m voxel filter to ensure uniform point density.
The reduced set selection threshold is $\mathcal{E} = 0.25$, sorted set size factor is $F = 10$, and distance radius for the pose-based CEBC approximation function is $\alpha = 15$ across all results unless otherwise noted.
The descriptor neural network is trained on UCLA campus datasets not used in any results for 30 epochs.
A modified variant of \cite{chen2023direct} with loop closures and custom OptMap interfaces provides LiDAR scans with labeled poses and motion correction.
An odometry-only variant without loop closures is available open-source\footnote{Modified version of DLIO \cite{chen2023direct} for OptMap: \url{https://github.com/vectr-ucla/direct_lidar_inertial_odometry/tree/feature/optmap}}.

\subsection{Geometric Map Distillation Reward Function Evaluation}
\label{subsec:CEBC_result}

\begin{figure}[t!]
    \centering
    \includegraphics[width=0.48\textwidth]{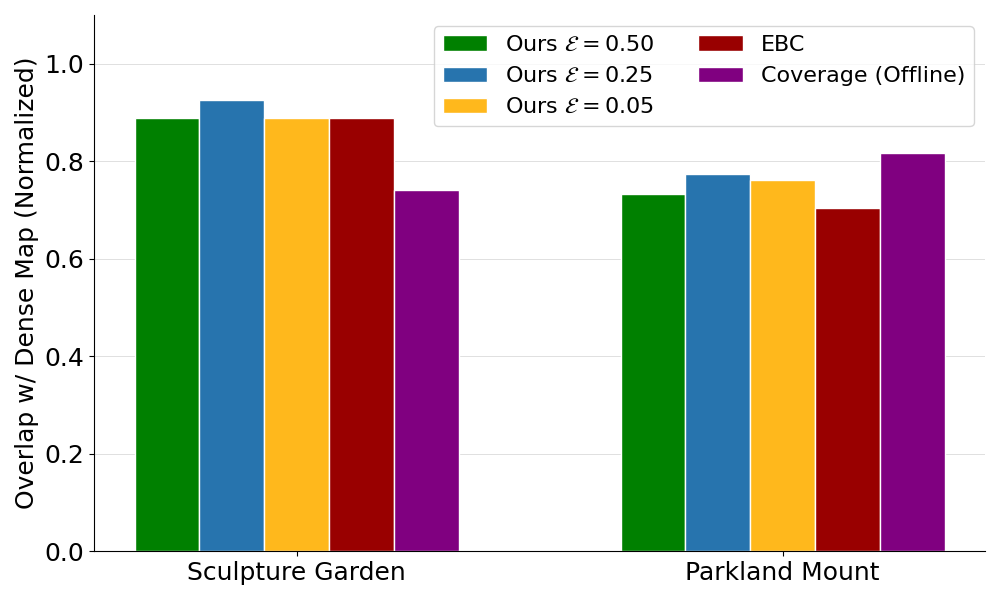}
    \caption{Normalized percent overlap of distilled maps for reward functions when maximized via the greedy algorithm, higher indicates better performance. Sculpture Garden is a custom dataset in a semi-forested environment, and Parkland Mount is from \cite{ramezani2020newer}.
    CEBC (Ours) is evaluated using reduced input sets constructed with admission thresholds of $\mathcal{E}=0.05, 0.25, $ and $0.5$.}
    \label{fig:result1_reward_comp}
    \vskip -0.1in
\end{figure}

The first experiment seeks to find the best function for quantifying informativeness in geometric map distillation and studies the effects of the reduced set selection threshold parameter $\mathcal{E}$.
We consider three possible variants: a direct implementation of descriptor coverage \cref{eq:coverage_desc_simple}, Exemplar-Based Clustering (EBC) \cref{eq:EBC}, and Continuous Exemplar-Based Clustering (CEBC, Ours) \cref{eq:CEBC}.
CEBC was further split into three variants with reduced set thresholds of $\mathcal{E}=0.05$, $0.25$, and $0.5$.
Each function was optimized using the greedy algorithm to isolate results from any bias present in streaming submodular optimization.
The datasets used were the semi-urban outdoor Sculpture Garden and Parkland Mount from \cite{ramezani2020newer}.
To measure informativeness, the overlap between each function's output map and a dense comparison map is used.
The number of scans used in output maps was increased until any output exceeded 50\% overlap.
The overlap percent for each output was then normalized with an optimal result generated by greedily selecting scans that overlap with the dense map.

Results from the first experiment are seen in \cref{fig:result1_reward_comp} where CEBC performs near the optimal across both datasets and with multiple reduced set thresholds, EBC performs well on Sculpture Garden, and coverage performs well on Parkland Mount.
Normalized overlap percentages for Sculpture Garden were 89\%, 93\%, 89\%, 89\%, and 74\% for CEBC with $\mathcal{E}=0.5$, $0.25$, $0.05$, EBC, and coverage respectively.
Normalized overlap for Parkland Mount was 73\%, 78\%, 76\%, 70\%, and 82\% for the five variants respectively.
Increasing the reduced set threshold from $\mathcal{E}=0.25$ to $0.5$ reduced overlap, suggesting that higher thresholds discard useful information.
However, CEBC with all reduced set thresholds equaled or exceeded EBC, indicating that the continuous trajectory assumption helped quantify informativeness.
CEBC also compares favorably with the coverage function which performed well on Parkland Mount but poorly on Sculpture Garden.
This inconsistency arose because the coverage function only rewards self-dissimilar solutions, making it less robust.
This resulted in lower overlap on the Sculpture Garden dataset which had more sequential descriptors further than 0.25 apart (48 vs. 32, \cref{fig:assumption1-context}).

\begin{figure}[t!]
    \centering
    \includegraphics[width=0.48\textwidth]{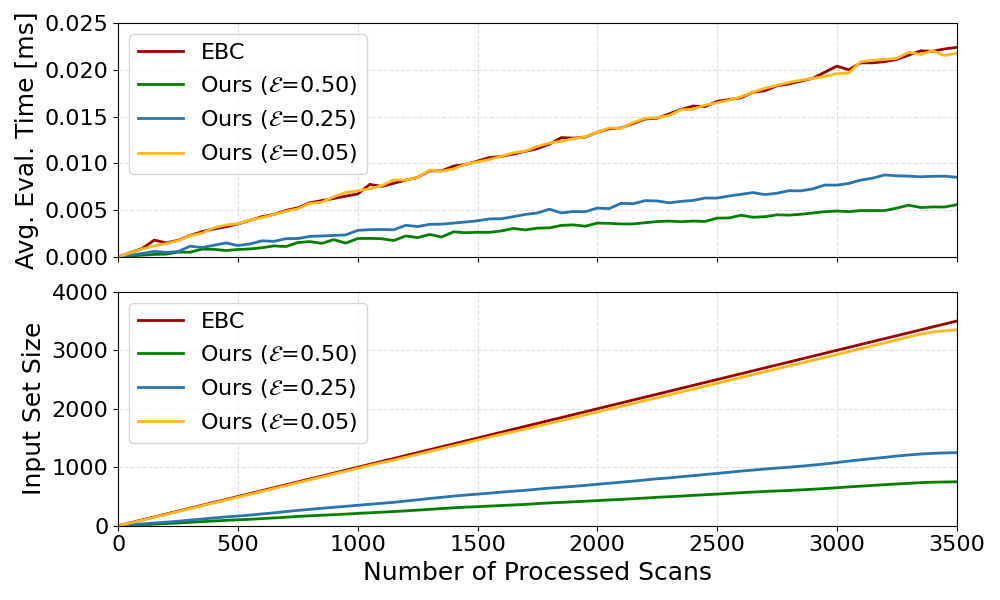}
    \caption{Average function evaluation time and input set size for Sculpture Garden dataset. Ours is CEBC. Evaluation time decreases with input set size. Multiple reduced set distance thresholds $\mathcal{E}$ are used to show that computation time approaches EBC as the reduced and full sets converge.}
    \label{fig:result1_comp_time}
    \vskip -0.15in
\end{figure}

\begin{figure*}[t]
    \centering
    \vspace{6pt}
    \includegraphics[width=0.96\textwidth]{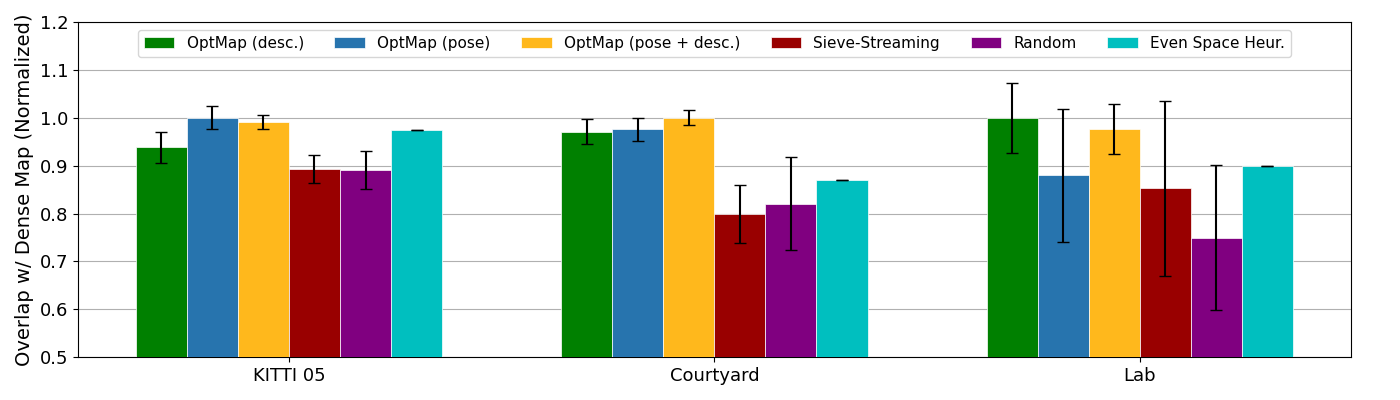}
    \caption{Percent of overlap from dense map to output map for ablation test on KITTI 05 \cite{behley2019iccv}, Courtyard, and Lab datasets shown in \cref{fig:result2_densemaps}. Results are normalized by the highest scoring method for each dataset. KITTI 05 output maps made with fifty} scans, Courtyard maps made with twelve scans, and Lab maps made with eight scans.
    \label{fig:overlap_ablation}
    \vskip -0.15in
\end{figure*}

\begin{figure}[t!]
    \centering
    \includegraphics[width=0.48\textwidth]{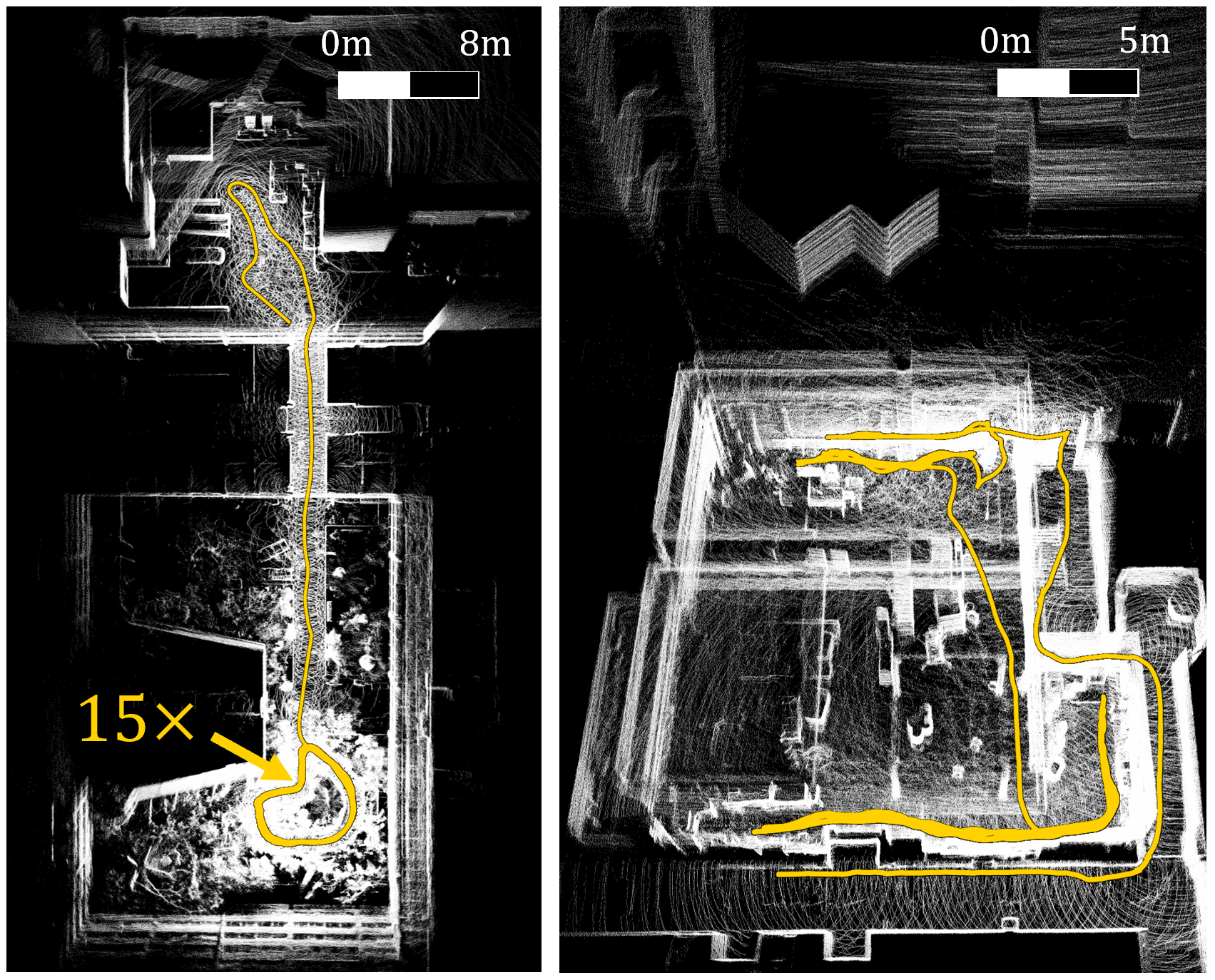}
    \caption{Dense maps used for \cref{fig:overlap_ablation}. Each map is made from 250 scans from a 32-beam Ouster OS0, with the sensor trajectory indicated by the yellow path. Left is Courtyard which entails fifteen loops of an outdoor courtyard followed by a single mapping pass through tight corridors and an adjacent alley in the figure top. Right is Lab which contains multiple mapping passes capturing the interior of a large lab space and a single pass captured the outdoor alley and nearby tight hallways.}
    \label{fig:result2_densemaps}
    \vskip -0.2in
\end{figure}

CEBC also uses a reduced input set to reduce function evaluation time as shown in \cref{fig:result1_comp_time}.
For CEBC with $\mathcal{E}=0.5$, $0.25$, $0.05$ and EBC the final input set size for the Sculpture Garden was 753, 1256, 3376, and 3575 scans respectively.
The final average evaluation time for each variant was 0.0057, 0.0087, 0.0219, and 0.0228 milliseconds showing a strong correlation between input set size and function evaluation time.
Compared to EBC, CEBC with the best threshold from \cref{fig:result1_reward_comp}, $\mathcal{E}=0.25$, saw an 84\% reduction in average evaluation time and 65\% reduction in input set size.
Note that the greedy algorithm requires $O(kN)$ evaluations to find a solution, so a solution with 25 elements would require 2038 ms of optimization for EBC and 273 ms for CEBC with $\mathcal{E}=0.25$.

\subsection{Dynamic Reordering Effectiveness Ablation Study}
\label{subsec:dyn_reorder_result}

\begin{figure*}[t!]
    \centering
    \includegraphics[width=0.96\textwidth]{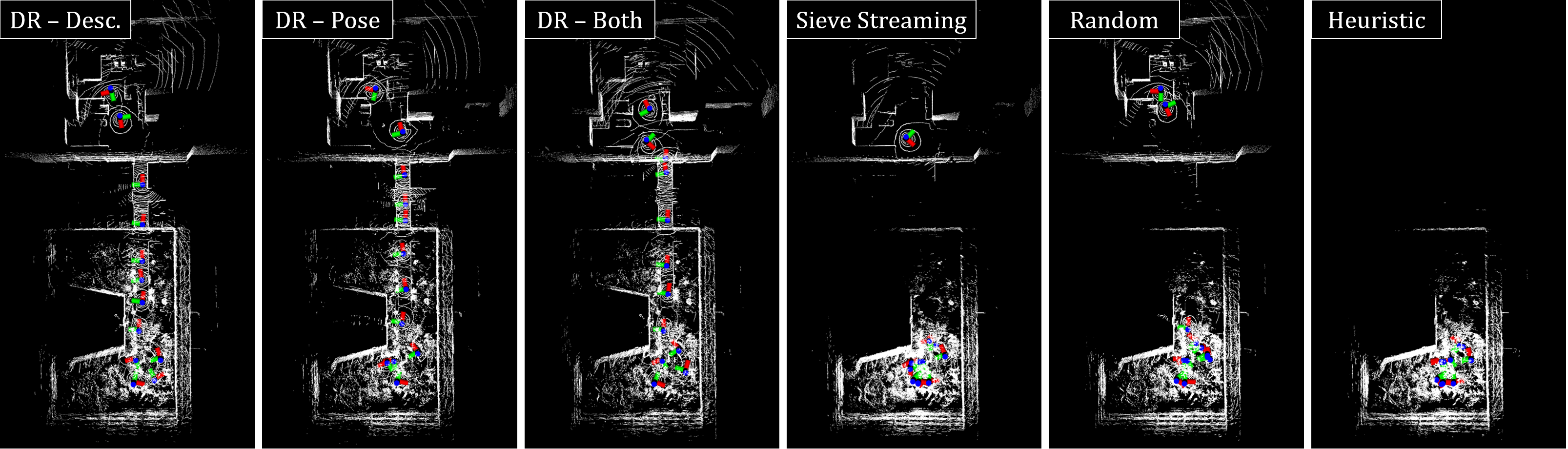}
    \caption{Sample output maps using different methods from \cref{fig:overlap_ablation}. Methods from left to right are OptMap with descriptor-based dynamic reordering, OptMap with pose-based dynamic reordering, OptMap with both descriptor and pose dynamic reordering, Sieve-Streaming, random selection, and the evenly spaced heuristic. Dynamic reordering methods capture the entire map where the other three methods only select scans from the loop seen at the bottom of each map.}
    \label{fig:result2_boelter}
    \vskip -0.1in
\end{figure*}

\begin{figure}[t!]
    \centering
    \includegraphics[width=0.48\textwidth]{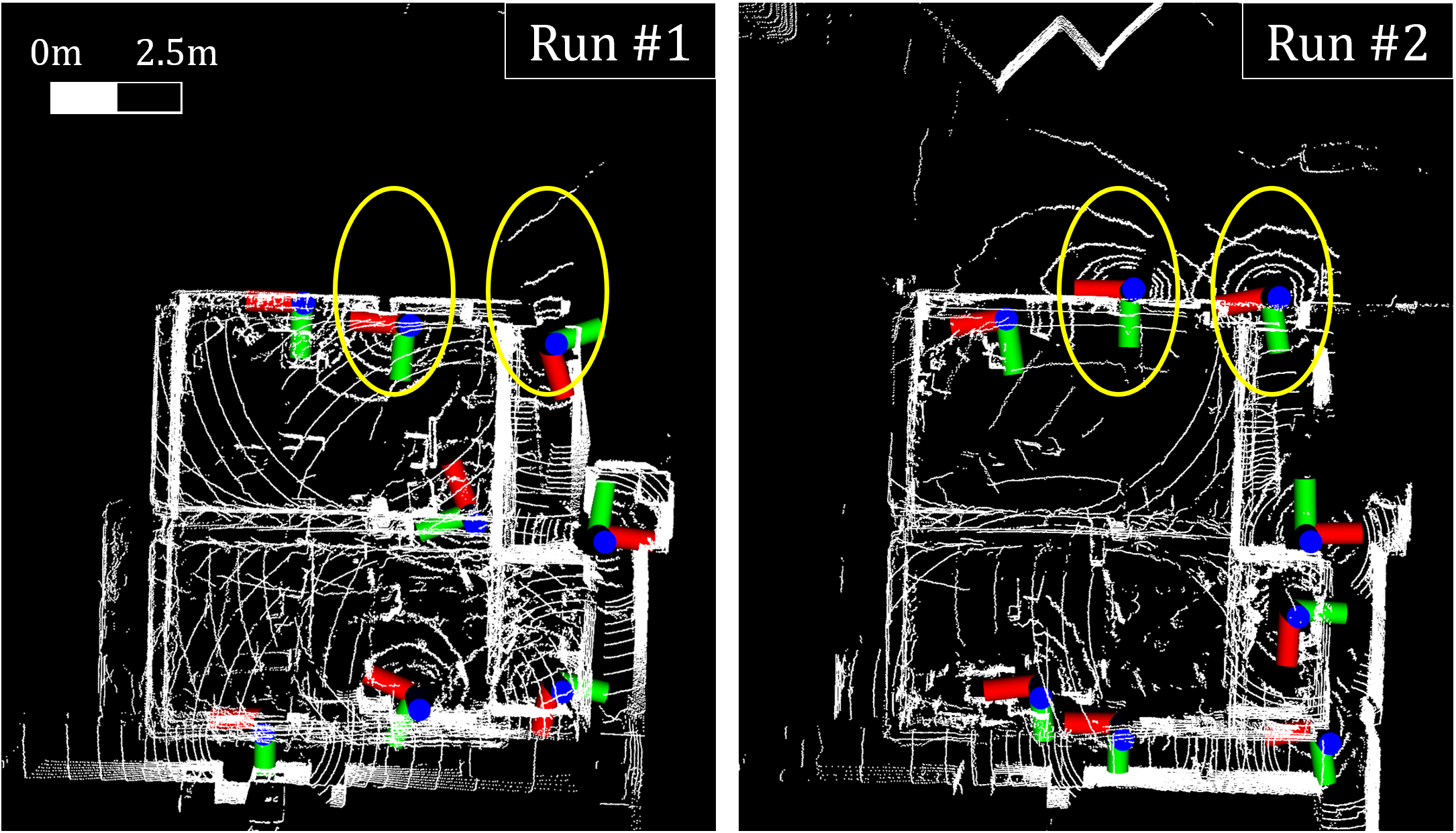}
    \caption{The pose-based dynamic reordering is unable to accurately select the highest value elements in cases where nearby scans capture significantly different environments. These two maps from the Lab dataset (\cref{fig:overlap_ablation}) were generated using the same pose-based approximation dynamic reordering method, and demonstrate the large standard deviation from the pose-based dynamic reordering as scans are arbitrarily selected from one side of the wall (yellow ovals) leading to occasional suboptimal results.}
    \label{fig:result2_lab_pose_comp}
    \vskip -0.2in
\end{figure}

The second experiment highlights the benefits of dynamic reordering and compares the relative strengths and weaknesses of multiple approximation functions used to generate ordering scores.
In addition to the approximations presented in \cref{sec:CEBC_approx} (the pairwise information descriptor approximation \cref{eq:coverage_step3} and heuristic pose approximation \cref{eq:pose_heur}), we also evaluate a third approximation, which takes the sum of both the pairwise information descriptor and pose approximations. 
An ablation study compares OptMap using these three approximation functions with the baseline Sieve-Streaming algorithm, random element selection, and the heuristic solution with even spacing described by \cref{alg:heur_sol}.
Sieve-Streaming is OptMap with no dynamic reordering but has identical initial solution value bounds.
All methods use input sets filtered by \cref{alg:reduced_set_selection}.

Each optimization method generated ten distilled maps on KITTI 05 \cite{behley2019iccv} and two custom datasets collected at UCLA with dense maps shown in \cref{fig:result2_densemaps}.
KITTI 05 is a medium length dataset (2736 scans) collected on a fast moving car with no significant bias caused by remapping. 
Courtyard (\cref{fig:result2_densemaps} left) introduces a bias to the input set by repeating a small loop fifteen times before making a single mapping pass away from the loop.
The bias in the input set is the result of the loop being significantly overrepresented, with the loop accounting for 86\% of collected scans but only 31\% of the unique distance traveled.
Finally, Lab (\cref{fig:result2_densemaps} right) introduces input set bias by mapping the inside south and north walls of a lab space with four passes each.
Scans in the Lab dataset were collected on either side of a thin wall and garage door to stress the position-based CEBC approximation.
The UCLA datasets are meant to show the value of dynamic reordering by biasing the input sets.
Although this bias is uncommon in open-source datasets, it is not uncommon in practical applications where mapping sessions revisit the same location several times or remain stationary.
The output map size for each dataset is hand-picked, with fifty scans used for KITTI 05, twelve scans used for Courtyard, and eight scans used for Lab.

The average overlap percent and its standard deviation for all three datasets are presented in \cref{fig:overlap_ablation}. 
We normalized the results by the highest performing method for each dataset.
All methods are expected to perform well on KITTI 05, but the fast moving car and relatively self-similar street environment stressed the descriptor-based approximation method.
The descriptor, pose, and combined dynamic reordering methods have normalized average overlaps of 93.8\%, 100\%, and 99.0\%, with normalized standard deviations of 3.2\%, 2.4\%, and 1.5\% respectively.
Sieve-Streaming, random, and the evenly spaced heuristic had normalized average overlaps of 89.3\%, 89.1\%, and 97.4\% and standard deviations of 2.9\% and 3.9\% (the heuristic is deterministic and hence has no variance).
We note that the $1/2$-suboptimality bound is reached on the UCLA datasets using only 1-3 scans, so Sieve-Streaming performs similarly to random selection in all cases.

Courtyard is meant to demonstrate the strengths of dynamic reordering, which are emphasized by the visual results presented in \cref{fig:result2_boelter}.
The first three panels of \cref{fig:result2_boelter} are representative maps from the descriptor, pose, and combined dynamic reordering methods where multiple scans and sufficient coverage of the single pass are present for all three methods.
The final three panels show representative maps from Sieve-Streaming, random, and the evenly spaced heuristic; each contain very few scans and minimal coverage from the single pass.
The qualitative results match the overlap and standard deviations presented in \cref{fig:overlap_ablation}.
The average normalized overlaps are 97.1\%, 97.6\%, and 100\% for the descriptor, pose, and combined approximation dynamic reordering, and the average normalized standard deviations are 2.5\%, 2.4\%, and 1.6\%.
In contrast, the average normalized overlaps for Sieve-Streaming, random, and the evenly spaced heuristic are 80.0\%, 82.0\%, and 87.0\%.
The average normalized standard deviations for Sieve-Streaming and random are 6.0\% and 9.7\%.
Both \cref{fig:result2_boelter} and \cref{fig:overlap_ablation} show that the value of dynamic reordering increases when an input set contains redundantly mapped areas.
In the context of field operations, this is most apparent in providing summary maps when a robot has patrolled the same area many times, but then explores a new environment, i.e., a patrol task followed by a reconnaissance task. 

\begin{figure}[t!]
    \centering
    \includegraphics[width=0.48\textwidth]{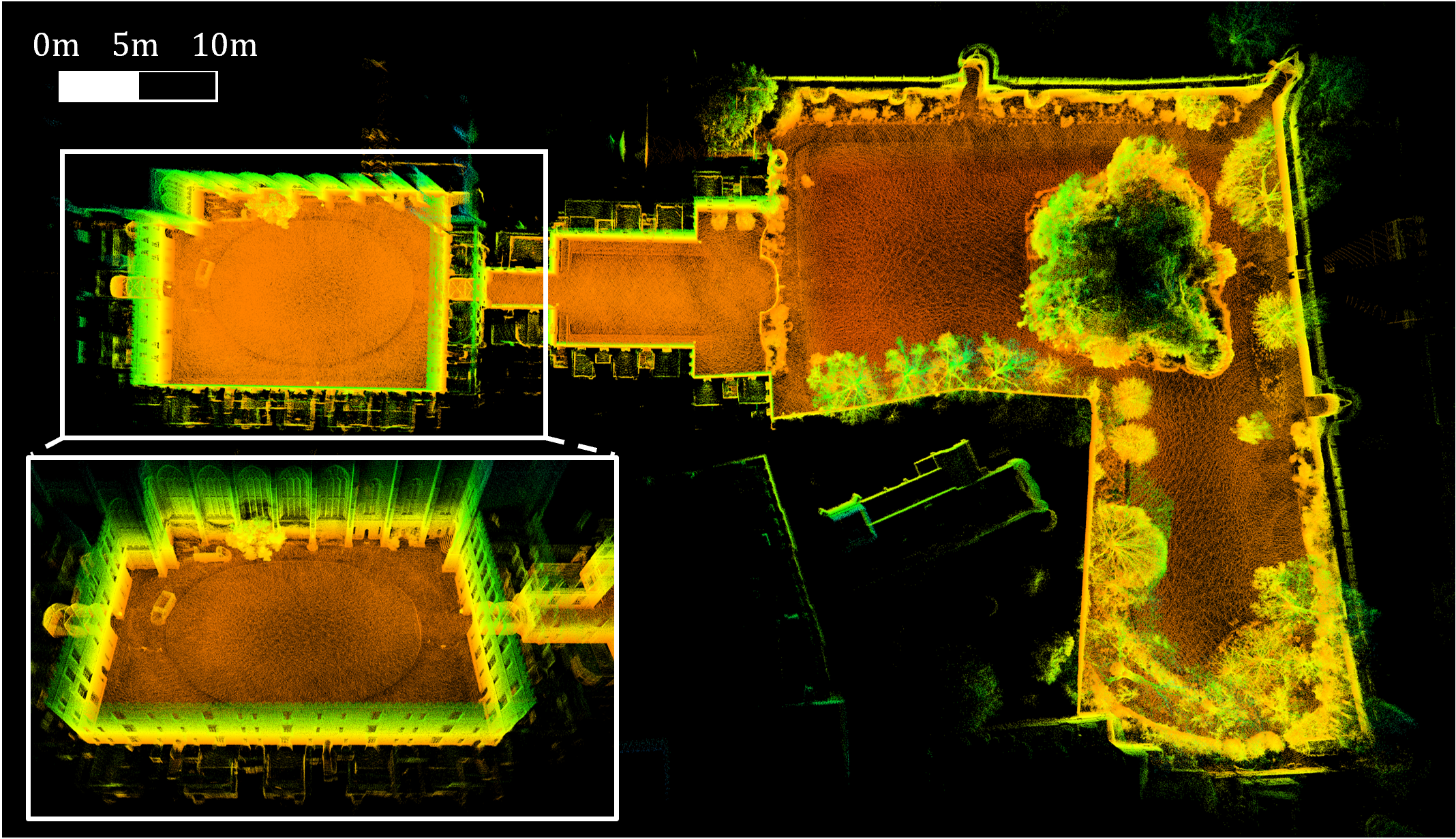}
    \caption{250 scan output OptMap map from Newer College Dataset Long. Color indicates z-coordinate of point. Inset box shows the level of detail and geometric map consistency which OptMap can produce. Maps with similar level of detail are difficult to maintain in online SLAM due to memory or computation time constraints.}
    \label{fig:result3_NCD}
    \vskip -0.15in
\end{figure}

The Lab dataset highlights the strengths of the pairwise information descriptor approximation for dynamic reordering, as it preserves scene context when scans from different environments are collected in close proximity.
The average normalized overlap percents for the descriptor, pose, and combined methods are 100\%, 88.0\%, and 97.7\% with standard deviations of 7.4\%, 13.9\%, and 5.3\% respectively.
The pose approximation had a significantly worse average and higher standard deviation because scans from either side of a thin wall will be scored similarly despite capturing significantly different environments, as shown in \cref{fig:result2_lab_pose_comp}.
This indicates that descriptors are best in structured environments where small pose differences can lead to large scene changes, and pose approximation is best in outdoor mapping where descriptors are less capable of learning unstructured environments.
Sieve-Streaming and random had low average normalized overlap with 85.3\% and 75.0\% with particularly high standard deviation with 18.3\% and 15.1\%.
The evenly spaced heuristic had a normalized overlap of 89.9\%, which matches the trend of the evenly spaced heuristic outperforming Sieve-Streaming, Random, and either the descriptor or pose approximation \emph{in datasets that stress the approximation}.
Across all three datasets the combined approximation is best or near best, indicating that it preserves the relative strengths of the pose and descriptor approximation without suffering in edge cases where those particular approximations fail.

\begin{table*}
\centering
\caption{OptMap computation time results. Left column gives the dataset, where sequences from three standardized datasets were chosen alongside one custom dataset collected on a 128-beam Ouster OS1 LiDAR (Forest Loop pictured in \cref{fig:result3_forestloop}). Distilled map size indicates how many scans were used to generate the map, optimization time is the time from function call to when the solution set is selected, and load time is the remaining time spent loading point clouds from storage to computer memory. Method (1) uses the initial solution from \cref{subsec:bounds} and method (2) does not.}
\label{tab:result2_output}
\begin{tabular}[c]{ p{1.45in}||p{0.55in}||p{0.55in}|p{0.55in}|p{0.55in}|p{0.55in}|p{0.55in}|p{0.55in}}
\centering Dataset \\(Num. Scans) & \centering Distilled Map Size & \centering (1) Opt. [ms] & \centering (1) Load [ms] & \centering \textbf{(1) Total [ms]} & \centering (2) Opt. [ms] & \centering (2) Load [ms] & \centering \textbf{(2) Total [ms]} \arraybackslash\\
\hline
\multirow{3}{1.5in}{\centering Oxford Spires Obs. Quarter Seq. 1 \cite{tao2025spires} \\ (2,894 scans)} & \centering25 & \centering2.2 & \centering 65.2 & \centering\textbf{67.4} & \centering14.0 & \centering 68.0 & \centering\textbf{82.0} \arraybackslash\\
& \centering50 & \centering4.1 & \centering 123.6 & \centering\textbf{127.7} & \centering28.9 & \centering 131.1 & \centering\textbf{160.0} \arraybackslash\\
& \centering100 & \centering8.4 & \centering 243.9 & \centering\textbf{252.3} & \centering57.6 & \centering 252.2 & \centering\textbf{309.8} \arraybackslash\\
\hline
\multirow{3}{1.5in}{\centering Oxford Spires Blenheim Palace Seq. 3 \cite{tao2025spires} \\(6,970 scans)} & \centering25 & \centering3.8 & \centering 66.8 & \centering\textbf{70.6} & \centering23.0 & \centering 67.7 & \centering\textbf{90.7} \arraybackslash\\
& \centering50 & \centering7.0 & \centering 126.4 & \centering\textbf{133.4} & \centering46.5 & \centering 132.6 & \centering\textbf{179.1} \arraybackslash\\
& \centering100 & \centering15.6 & \centering 252.6 & \centering\textbf{268.2} & \centering93.1 & \centering 256.4 & \centering\textbf{349.5} \arraybackslash\\
\hline
\multirow{3}{1.5in}{\centering Semantic KITTI 10 \cite{behley2019iccv} \\(1,224 scans)} & \centering25 & \centering2.2 & \centering 65.2 & \centering\textbf{67.4} & \centering14.0 & \centering 68.0 & \centering\textbf{82.0} \arraybackslash\\
& \centering50 & \centering4.1 & \centering 123.6 & \centering\textbf{127.7} & \centering28.9 & \centering 131.1 & \centering\textbf{160.0} \arraybackslash\\
& \centering100 & \centering8.4 & \centering 243.9 & \centering\textbf{252.3} & \centering57.6 & \centering 252.2 & \centering\textbf{309.8} \arraybackslash\\
\hline
\multirow{3}{1.5in}{\centering Semantic KITTI 02 \cite{behley2019iccv} (4,663 scans)} & \centering100 & \centering26.9 & \centering 405.7 & \centering\textbf{432.6} & \centering67.1 & \centering 411.5 & \centering\textbf{478.6} \arraybackslash\\
& \centering250 & \centering229.8 & \centering 859.4 & \centering\textbf{1089.2} & \centering642.5 & \centering 660.7 & \centering\textbf{1303.2} \arraybackslash\\
& \centering500 & \centering329.6 & \centering 1854.9 & \centering\textbf{2184.5} & \centering2224.0 & \centering 803.1 & \centering\textbf{3027.1} \arraybackslash\\
\hline
\multirow{3}{1.5in}{\centering Newer College Long Experiment \cite{ramezani2020newer} \\(26,559 scans)} & \centering100 & \centering80.5 & \centering 363.6 & \centering\textbf{444.1} & \centering493.8 & \centering 354.4 & \centering\textbf{848.2} \arraybackslash\\
& \centering250 & \centering186.5 & \centering 896.5 & \centering\textbf{1083.0} & \centering1248.8 & \centering 917.5 & \centering\textbf{2166.3} \arraybackslash\\
& \centering500 & \centering366.7 & \centering 1919.7 & \centering\textbf{2286.4} & \centering2584.4 & \centering 1956.8 & \centering\textbf{4541.2} \arraybackslash\\
\hline
\multirow{3}{1.5in}{\centering Forest Loop \\ (34,158 scans)} & \centering100 & \centering103.8 & \centering 561.6 & \centering\textbf{665.4} & \centering683.8 & \centering 570.0 & \centering\textbf{1253.8} \arraybackslash\\
& \centering250 & \centering275.4 & \centering 1389.9 & \centering\textbf{1665.3} & \centering1694.4 & \centering 1405.4 & \centering\textbf{3099.8} \arraybackslash\\
& \centering500 & \centering531.0 & \centering 2672.3 & \centering\textbf{3203.3} & \centering3353.4 & \centering 2712.4 & \centering\textbf{6065.8} \arraybackslash\\
\end{tabular}
\vskip -0.05in
\end{table*}

\begin{figure*}[t!]
    \centering
    \includegraphics[width=0.96\textwidth]{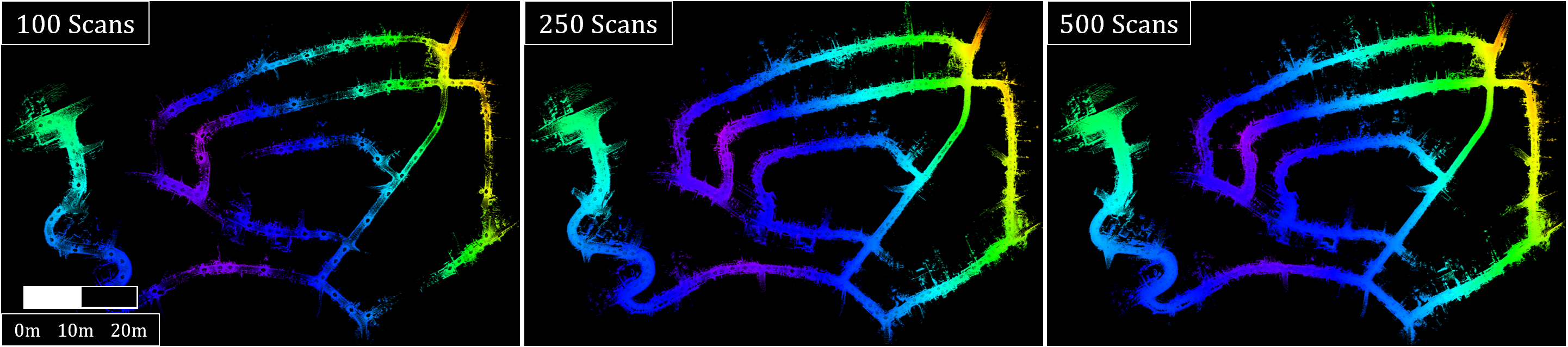}
    \caption{Output OptMap maps for Semantic KITTI 02 with 100, 250, and 500 scans. Color indicates z-coordinate of point. OptMap is capable of generating map summaries with flexible size and resolution for autonomous planning at various scales.}
    \label{fig:result3_kitti}
    \vskip -0.05in
\end{figure*}

\begin{figure}[t!]
    \centering
    \includegraphics[width=0.48\textwidth]{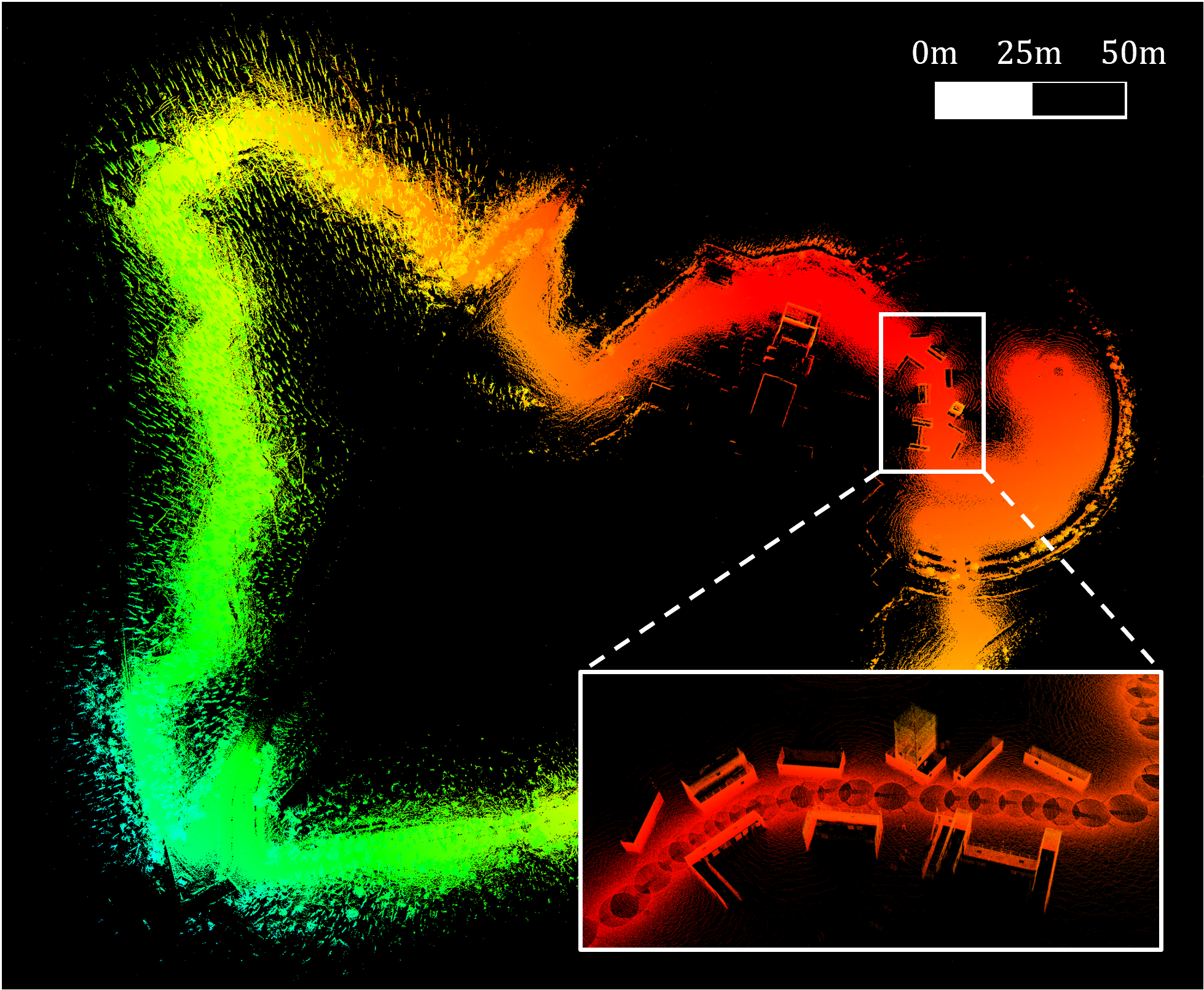}
    \caption{500 scan output OptMap map from custom forest loop dataset. Color indicates z-coordinate of point. Forest loop dataset covers 2.6 km and collects 34,158 scans. Inset box shows the level of detail around several buildings.}
    \label{fig:result3_forestloop}
    \vskip -0.15in
\end{figure}

\subsection{Real-Time Map Summarization Evaluation}
\label{subsec:summarization_res}

\begin{figure*}[t!]
    \centering
    \includegraphics[width=0.98\textwidth]{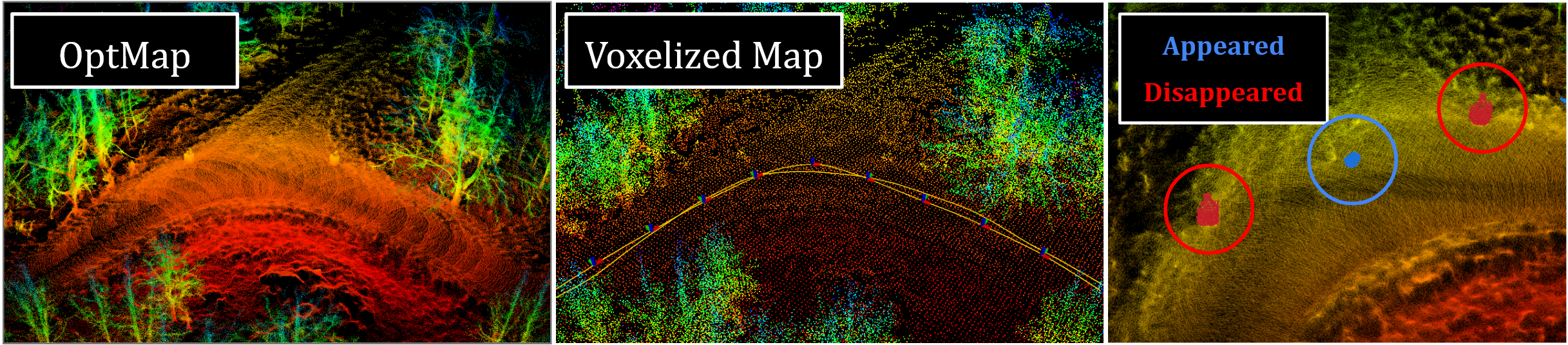}
    \caption{OptMap before map using 60 scans (left), voxelized global map from the corresponding SLAM algorithm (middle), and change detection results using OptMap (right). Points that appeared are shown in blue and points that disappeared are shown in red.}
    \label{fig:change_detection}
    \vskip -0.15in
\end{figure*}

The third experiment tests OptMap's return time on large LiDAR datasets with and without \emph{a priori} bounds.
By presenting the time to find and return actionable maps, we also highlight OptMap's online map distillation capability in the most challenging scenarios with large input set sizes.
This experiment distills three small datasets using maps with 25, 50, and 100 scans: Oxford Spires Observatory Quarter Sequences 1 and 3 \cite{tao2025spires}, and Semantic KITTI Sequence 10 \cite{behley2019iccv}.
Three larger datasets are distilled using 100, 250, and 500 scans: Semantic KITTI Sequence 02, the Newer College Dataset Long Experiment \cite{ramezani2020newer}, and Forest Loop, a custom dataset collected in an unstructured forest environment.
Load times are directly influenced by the number of points in each scan, which in turn is determined by the number of beams in the sensor.
Oxford Spires used a 64-beam Hesai QT64, KITTI used a 64-beam Velodyne HDL-64E, Newer College used a 64-beam Ouster OS1, and Forest Loop used a 128-beam Ouster OS1.
The two largest sessions are Newer College Dataset's Long Experiment \cite{ramezani2020newer} and Forest Loop, which contain 26,559 and 34,158 scans representing 44 and 28.5 minutes (or 18.6 and 34 GB) of mapping respectively.
OptMap was run five times for each configuration with negligible variation ($<20$ ms total max) in either the optimization or map loading time.

The optimization, map loading, and total times for all six datasets are presented in \cref{tab:result2_output}.
The third through fifth columns (labeled method 1) used \emph{a priori} bounds and the sixth through eighth (labeled method 2) did not.
The percent of total time saved by the \emph{a priori} bounds across the three smaller datasets was between 19$\%$-24$\%$, but the average percent saved across the three larger was 18$\%$-49$\%$.
The dramatic increase in time saved is due to the fact that tighter bounds reduce the number of solution sets that need to be maintained for larger input sets.
The total time increases linearly with the distilled map size due to linear increases in both optimization and map loading times.
Semantic KITTI is a challenging dataset for OptMap because it is collected on a relatively fast moving car.
This stresses the assumption that sequential descriptors are similar and thus results in larger input sets.

Summarization is a key application for OptMap, with visual results presented in \cref{fig:result3_kitti,fig:result3_NCD,fig:result3_forestloop}.
\Cref{fig:result3_kitti} demonstrates generated summary maps with flexible size, where size-constrained summarization could be valuable for bandwidth-aware multi-robot map sharing or global place recognition.
\Cref{fig:result3_NCD} shows a high-fidelity and extremely detailed map of the Newer College Dataset Long Experiment generated by OptMap.
The inset shows a detailed map that is not usually available from LiDAR SLAM due to the computation and memory load required for maintaining dense maps.
Finally, \cref{fig:result3_forestloop} shows the largest OptMap distilled map by number of input scans.
We separate optimization and map loading times in \cref{tab:result2_output} to highlight the fact that a significant portion of OptMap's total time comes from needing to load LiDAR scans.
This demonstrates the need for indirect optimization (i.e., no point-to-point computations) in geometric map distillation, which is further highlighted in the following result.

\subsection{Application: Geometric Change Detection}
\label{subsec:change_detect}

The final result highlights a practical application of OptMap for online change detection.
We assume a human operator has identified a particular location that has been mapped multiple times to check for objects that have appeared, disappeared, or moved.
Changes can be identified by comparing two point clouds from different times, where the distance from each point in the before (resp. after) map to its nearest neighbor in the after (resp. before) map is used to determine if the point is a change.
Results are better when the two maps are dense (i.e., the average distance between points is small) and thus change detection is usually handled by offline algorithms such as \cite{kim2022lt}. 
However, for patrol or search-and-rescue missions, online change detection may be very pertinent for mission planning (e.g., a set of fresh footprints are detected).
Although some geometric change detection algorithms are real-time \cite{qian2022pocd, wellhausen2017reliable, krawciw2024lasersam, krawciw2023changeofscenery}, they only compare a single sensor output to a map and therefore lack the density to detect subtle details.
A full comparison is beyond the scope of this work, but to the best of our knowledge, there is no current online geometric change detection algorithm which utilizes dense before and after maps.

A dense before submap generated by OptMap from a junction in a dataset collected at the Army Research Laboratory is shown in \cref{fig:change_detection} with the before, after, and difference map shown in \cref{fig:top-right-changedetect}.
The operator defined the position and time bounds for the before and after map given knowledge of the mapping session.
Points in either map are labeled as a potential change if the nearest point in the other map is at least 10 cm away.
Potential changes are then pruned by a simple noise filter that discards points where the average distance to the nearest 50 potential change points is more than 5 cm away.
This noise filter is meant to discard spurious incorrect changes that might result from mismatched scan lines or field of view differences.
The results of change detection are shown on the right of \cref{fig:change_detection}, where two 60 cm tall cases were identified as negative changes, and a 20 cm tall backpack was identified as a positive change with no incorrect detections.
The time to generate each OptMap output with 60 scans was just 196 ms.

This result is remarkable because the change detection maps cannot be generated using a simple downsampling method (i.e., voxelization).
The left image in \cref{fig:change_detection} shows the OptMap before map and center shows a voxelized dense map (with resolution 0.25 m).
This resolution is selected such that the full map for a mapping session up to approximately 5 km could be stored in the available 16 GB of computer memory.
The before and after submaps each contain approximately fifteen million points in twenty meters traveled.
With approximately 40 bytes per point, the memory burden of saving a global map with equal density would be approximately 15 MB per meter traveled.
Such a rate is not feasible as even short sessions that traverse 1 km would demand nearly 15 GB of memory.
Current research indicates a push for lightweight, memory optimized SLAM algorithms that save fewer points \cite{thorne2025submodular, dong2025lidar}, so OptMap is positioned to supplement these algorithms by providing flexible, online, and application-specific maps such as those shown in \cref{fig:change_detection}.

%% file: sections/conclusion.tex
\section{Conclusion}
\label{sec:conclusion}

This work presented OptMap, an online geometric map distillation algorithm that can generate provably near-optimal representative distillation maps from large LiDAR mapping datasets.
We propose multiple theoretical and algorithmic innovations which provide a means for quantifying informativeness in geometric maps, and allows for the efficient optimization of such maps.
The novel submodular reward function Continuous Exemplar-Based Clustering (CEBC) effectively rewards informativeness with the added ability to significantly reduce the size of input sets and improve computation time via the elimination of sequentially redundant elements.
Dynamic reordering is a novel addition to streaming submodular maximization algorithms which uses multiple approximation functions for CEBC to address input order bias.
Tight \emph{a priori} solution bounds further reduce computation time, and basic position and time constraints are introduced to make OptMap immediately useful to several applications including online change detection and multi-robot map sharing.
The results demonstrate the efficacy of CEBC, dynamic reordering, and tight \emph{a priori} bounds across a range of custom and open-source LiDAR datasets.
We demonstrate a practical application for OptMap in online change detection and envision several other similar novel capabilities that could be derived from OptMap.

Future work on OptMap can incorporate higher quality learned descriptors and other similarity metrics for improved distillation results across a wider variety of environments.
Although the descriptors in this work were shown to accurately describe a LiDAR scan, descriptors that better discriminate based on local features (i.e., groups of points from the same object) could enable more sophisticated submodular functions for selecting localization submaps or sharing relevant scans without knowing relative poses.
Additional work in generalizing dynamic reordering could yield a novel generalized streaming submodular maximization algorithm that uses limited \emph{a priori} information to address significant forms of input order bias in computationally efficient methods.
Improved functionality could also be derived from general matroid constraints.
For example, a knapsack constraint could enable input sets with multiple contributing sensor types, and set coverage constraints could allow for minimum size solutions which satisfy a defined coverage criterion.
Such innovations can lead to a tighter, bidirectional integration between geometric mapping and the downstream autonomy algorithms or operators that rely on compact and informative geometric maps.

%% file: sections/Appendix.tex
\section*{Appendix}
\label{sec:appendix}

\begin{proof}[Proof of \cref{cor:outliers}]
    For each outlier $r_{i}$ and its corresponding source set $e_{j}$, $j \in \eta_{i}$, the difference in the distance to an arbitrary solution is $|d(r_{i},S)-d(e_{j},S)| \leq |2-0| \leq 2\mathcal{E}+2$ because the maximum distance between points on a hypersphere is $2$.
    Without loss of generality, assume that the $N$ outliers are the final $N$ elements of the reduced set, where elements of the full input set may still exist between these outliers.
    Restating the third line of \cref{theorem:similarity_of_scores} then separating the sum over the reduced set into inliers and outliers is then
    \begin{align*}
        |\Gamma_{E}(S) &- \Gamma_{R}(S)| \\  
        &= \frac{1}{d_{tot}}\sum_{i=1}^{|R|}\Bigl|\sum_{j \in \eta_{i}}w_{j}(d(e_{j},S) - d(r_{i},S))\Bigr|, \\
        &= \frac{1}{d_{tot}} \sum_{i=1}^{|R|-N-1}\Bigl|\sum_{j \in \eta_{i}}w_{j}(d(e_{j},S) - d(r_{i},S))\Bigr| \\
        &\hphantom{=} + \frac{1}{d_{tot}} \sum_{i=|R|-N}^{|R|}\Bigl|\sum_{j \in \eta_{i}}w_{j}(d(e_{j},S) - d(r_{i},S))\Bigr|,
    \end{align*}
    where we can apply \cref{lem:triangle} to the first $|R|-N-1$ elements and $|d(r_{i},S)-d(e_{j},S)| \leq 2\mathcal{E}+2$ to the final $N$ outliers.
    This results in
    \begin{align*}
        |\Gamma_{E}(S) &- \Gamma_{R}(S)| \\ 
        &\leq \frac{1}{d_{tot}} \sum_{i=1}^{|R|-N-1}\Bigl|\sum_{j \in \eta_{i}}2w_{j}\mathcal{E}\Bigr| \\
        &\hphantom{=}+ \frac{1}{d_{tot}} \sum_{i=|R|-N}^{|R|}\Bigl|\sum_{j \in \eta_{i}}w_{j}(2\mathcal{E}+2)\Bigr|, \\
        &= \frac{1}{d_{tot}} \sum_{i=1}^{|R|}\Bigl|\sum_{j \in \eta_{i}}2w_{j}\mathcal{E}\Bigr| + \frac{1}{d_{tot}} \sum_{i=|R|-N}^{|R|}\Bigl|\sum_{j \in \eta_{i}}2w_{j}\Bigr|.
    \end{align*}
    Using $\hat{w}_{i} = \sum_{j \in \eta_{i}}w_{j}$ and $\hat{w}_{i} \leq 2+\mathcal{E}$, we get
    \begin{align*}
        |\Gamma_{E}(S) - \Gamma_{R}(S)|\leq 2\mathcal{E} + \frac{2\mathcal{E}+4N}{d_{tot}}. 
    \end{align*}
    as desired. \qedhere
\end{proof}